\documentclass[sn-mathphys,Numbered]{sn-jnl}

\usepackage{graphicx}%
\usepackage{multirow}%
\usepackage{amsmath,amssymb,amsfonts}%
\usepackage{amsthm}%
\usepackage{mathrsfs}%
\usepackage[title]{appendix}%
\usepackage{xcolor}%
\usepackage{textcomp}%
\usepackage{manyfoot}%
\usepackage{booktabs}%
\usepackage{algorithm}%
\usepackage{algorithmicx}%
\usepackage{algpseudocode}%
\usepackage{listings}%

\usepackage{graphicx}
\usepackage{subcaption}
\usepackage{multirow}
\usepackage{hhline}
\usepackage{tabularx}

\newtheorem{theorem}{Theorem}
\newtheorem{lemma}{Lemma}

\newtheorem{assumption}{Assumption}
\usepackage{footnote}

\begin{document}

\title[Article Title]{Learning Ensembles of Interpretable Simple Structure\footnote{This manuscript is  under review at Annals of Operations Research, Special Issue on  Ensemble Learning for Operations Research and Business Analytics}}


\author*[1]{\fnm{Gaurav} \sur{Arwade}}\email{gbarwade@iastate.edu}

\author*[1]{\pfx{Dr.} \fnm{Sigurdur} \sur{Olafsson}}\email{olafsson@iastate.edu}

\affil*[1]{\orgdiv{Department of Industrial, Manufacturing and Systems, Engineering}, \orgname{Iowa State University}, \orgaddress{ \city{Ames}, \postcode{50010}, \state{Iowa}, \country{USA}}}

\abstract{
Decision-making in complex systems often relies on machine learning models, yet highly accurate models such as XGBoost and neural networks can obscure the reasoning behind their predictions. In operations research applications, understanding how a decision is made is often as crucial as the decision itself. Traditional interpretable models, such as decision trees and logistic regression, provide transparency but may struggle with datasets containing intricate feature interactions. 
However, complexity in decision-making stem from interactions that are only relevant within certain subsets of data. Within these subsets, feature interactions may be simplified, forming simple structures where simple interpretable models can perform effectively. 
We propose a bottom-up simple structure-identifying algorithm that partitions data into interpretable subgroups known as simple structure, where feature interactions are minimized, allowing simple models to be trained within each subgroup. We demonstrate the robustness of the algorithm on synthetic data and show that the decision boundaries derived from simple structures are more interpretable and aligned with the intuition of the domain than those learned from a global model. By improving both explainability and predictive accuracy, our approach provides a principled framework for decision support in applications where model transparency is essential.
}




\maketitle

\section{Introduction}\label{sec1}

Predictive analytics play an important role in quantitative support for decision making, but when the models are complex and the intuition behind the predictions is not easily explainable, decision makers may be more reluctant to rely on predictions made by the model. Many of the most successful predictive models fall within the domain of ensemble learning, but those models are often not very interpretable or explainable. In this work, we propose a new interpretable ensemble learning algorithm, a characteristic that we view as important for the application of predictive analytics in many operations research and business analytics domains.

All machine learning models attempt to learn an underlying universal distribution of the data.  Is the data best described by single distribution? Not always. Further, not all data points may be necessary for successful decisions, and understanding critical data points is thus key. In learning theory, the focus is on characterizing critical samples that define an unknown model \citep{ghadikolaei_learning_2019}.

The complexity of the data may arise from instances in the dataset being drawn from more than one distribution \citep{bishop_pattern_2006}.  While complex models can approximate a wide range of patterns, they may still miss some of the underlying local patterns. Most importantly, highly non-linear models such as neural nets and boosted trees can approximate any complex universal mapping function, but they lack interpretability. Similarly, a simpler model with interpretability will miss many local patterns due to bias. 
An alternative is to identify simple structures approximating underlying different data distributions and consequently learn simpler machine learning models for each. 

Understanding and leveraging the concept of simple structures in data can offer a balance between model complexity and interpretability. As illustrated in Figure \ref{fig:SS_boundry} a simple structure refers to subsets of data with underlying patterns that are simple to model, often characterized by linear separability. For example, figure \ref{fig:SS_boundry} depicts data composed of two normal distributions, each forming a simple structure.
The first structure consists of two linearly separable classes, Class 1 and Class 2. The second structure contains three linearly separable classes: Class 1, Class 2, and Class 3. 
When treated as a single dataset, applying a Logistic Regression model results in underfitting, as shown in figure \ref{fig:LR_Boundry}, due to the model’s inability to capture the complexity of the overall data distribution. 
Conversely, a Neural Network model achieves perfect accuracy, but its decision boundaries are unnecessarily intricate and difficult to interpret, as illustrated in figure \ref{fig:NN_Boundry}.

\begin{figure}[H]
     \centering
     \begin{subfigure}[b]{0.3\textwidth}
         \centering
         \includegraphics[width=\textwidth]{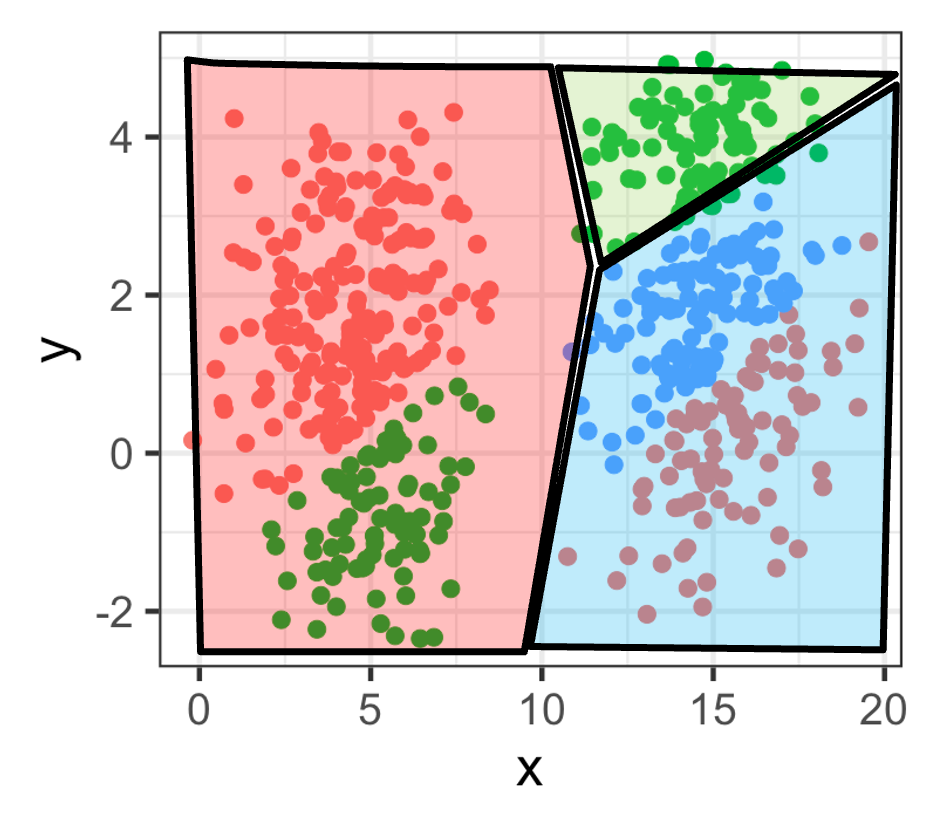}
         \caption{}
         \label{fig:LR_Boundry}
     \end{subfigure}
     \begin{subfigure}[b]{0.3\textwidth}
         \centering
         \includegraphics[width=\textwidth]{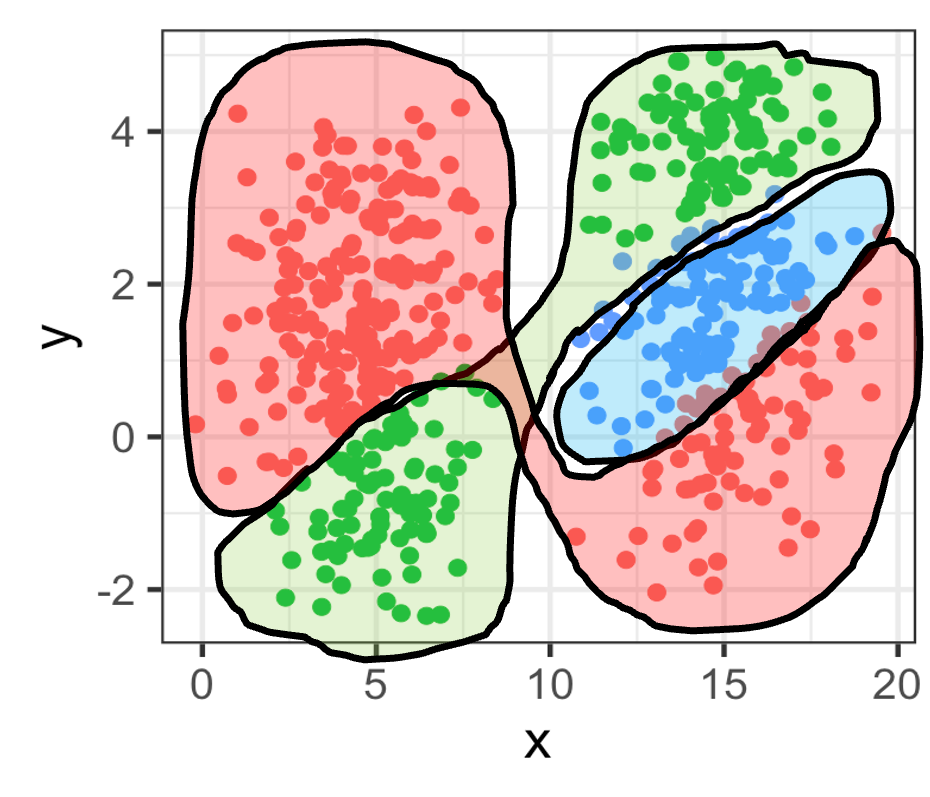}
         \caption{}
         \label{fig:NN_Boundry}
     \end{subfigure}
          \begin{subfigure}[b]{0.35\textwidth}
         \centering
         \includegraphics[width=\textwidth]{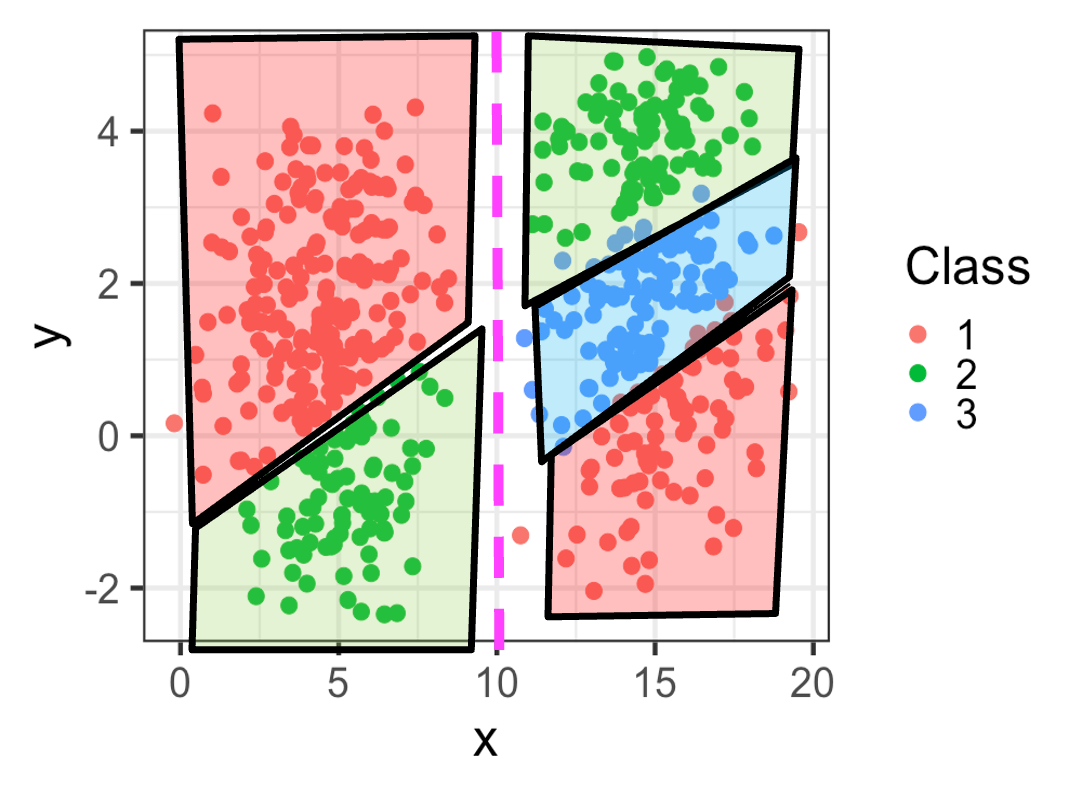}
         \caption{}
         \label{fig:SS_Boundry}
     \end{subfigure}
     \hfill
        \caption{Comparison of decision boundaries for data with underlying simple structures:
(a) Logistic Regression decision boundary underfits the data, failing to capture the underlying patterns
(b) Neural Network decision boundaries fit the data perfectly but are overly complex
(c) Ensemble of Logistic Regression models on the identified simple structures fits the data perfectly while maintaining simple and interpretable decision boundaries} \label{fig:SS_boundry}
\end{figure}

By acknowledging the presence of two simple structures within the dataset and fitting a separate Logistic Regression model for each structure, we can achieve the best of both worlds. As shown in Figure \ref{fig:SS_boundry}, the resulting decision boundaries are both simple and interpretable while achieving perfect accuracy, matching the performance of a Neural Network but with significantly reduced complexity. This highlights the value of identifying and utilizing simple structures in data for developing models that are both effective and interpretable.

Traditional state-of-the-art machine learning methods, including XGBoost, neural networks, support vector machines, and decision trees, are broadly classified as top-down methods. These methods construct decision boundaries by optimizing a global loss function, aiming to minimize errors across the entire dataset. While effective in achieving high accuracy, their global perspective often results in overly complex decision boundaries that obscure the natural structure of the data. This is particularly evident when the data contains subsets with simpler structures.

In contrast, our proposed approach adopts a bottom-up methodology, fundamentally differing from these top-down strategies. 
Our algorithm discovers localized subsets of instances where certain complex feature interactions cancel out, forming simple structures unique to those instances. By tailoring models to these naturally occurring subgroups, our method adapts to local patterns, capturing decision boundaries that emerge organically within each simple structure. This stands in contrast to top-down methods.

Even if a simple model on a simple structure does not significantly improve accuracy, if we know an instance to be classified belongs to one of the simple structures, this approach enables our algorithm to reduce complexity while preserving interpretability, achieving decision boundaries that are both meaningful and effective. Thus, besides selecting appropriate machine learning algorithms and their parameters, the training data optimization also plays a crucial role in the model's performance and interpretation.

This work introduces a simple structure-identifying algorithm to partition datasets into subsets with simpler underlying structures. By doing so, we enable the use of interpretable models, such as Logistic Regression and forming an ensemble to achieve comparable accuracy to complex models while maintaining transparency. The primary contributions of this paper are as follows:

\begin{itemize}
    \item  We define simple structures as data subsets with distinct patterns, such as linear separability, that can be modeled independently using simpler models. To ensure their identification, we outline key assumptions of the data with simple structures.
    \item We propose an algorithm that recursively identifies and partitions datasets into subsets with simpler structures. Recognizing the challenges posed by real-world data, we introduce heuristics to address violations of the simple structure definition and associated assumptions, making the algorithm robust to practical complexities.
    \item We demonstrate the algorithm’s effectiveness through comparative analysis with traditional approaches. Furthermore, we test the robustness of the simple structure-identifying algorithm against violations of the simple structure assumptions using synthetic data.
    \item We evaluate the algorithm's performance on open-source UCI and Kaggle datasets, showcasing how insights derived from models trained on simple structures differ from those obtained from models trained on the entire dataset.
\end{itemize}

In the remainder of this paper chapter 2 provides a relevant work, contrasting existing top-down methods with our proposed concept of simple structures. Chapter 3 formally introduces the concept of simple structures, defines the associated theoretical assumptions, and presents the proposed algorithm along with a proof of theoretical convergence. Chapter 4 addresses practical applications where theoretical assumptions of simple structures may be violated. 
Chapter 5 evaluates the algorithm’s performance on a range of real-world datasets, showcasing its ability to balance accuracy and interpretability effectively.

\section{Related Work}

We propose that data can be decomposed into multiple simple structures where complex feature interactions may cancel out, enabling models trained on these structures to make optimal decisions for instances within them. This perspective aligns with the broader concept of training data optimization or instance selection, which seeks to select or organize data to enhance model performance. 

Several instance-based learning methods, such as nearest neighbor rules \citep{hart_condensed_1968, gates_reduced_1972}, data density-based selection \citep{carbonera_density-based_2015}, and ranking-based approaches \citep{cavalcanti_ranking-based_2020, ghadikolaei_learning_2019}, aim to optimize training data by removing noise and redundancy \citep{olvera-lopez_review_2010}. Advanced techniques, like mixed-integer programming \citep{ghadikolaei_learning_2019} and local search algorithms \citep{lin_simultaneous_2021, neri_local_2020}, further compress datasets without sacrificing performance. As discussed, none of these methods explicitly try to model the underlying local simple structures but focus on improving accuracy and training time.

In machine learning, several methodologies incorporate the idea of approximating data using simple models, either directly or indirectly. Boosting, autoencoders, and collaborative filtering all use the idea of approximating the data using simple models or similar instances directly or indirectly.

Boosting attempts to exploit different structures in the data by learning a series of weak learners. 
The underlying idea is to combine weak learners to make a decision with outstanding performance~\citep{meir_introduction_2003},~\citep{freund_decision-theoretic_1997},~\citep{friedman_greedy_2001},~\citep{chen_xgboost_2016}.
Boosting attempts to fit close to the data, so the learned weak structures are dependent on the residuals from previous weak learners; thus, patterns may not be intuitive and lack model explainability~\citep{caruana_intelligible_2015}.  

The autoencoder compresses input into a meaningful latent representation and disentangles variations that are features~\citep{bank_autoencoders_2020},~\citep{bengio_representation_2013} akin to identifying simple structures in feature space as latent representations combine similar instances~\citep{chung_learning_2017},~\citep{tschannen_recent_2018}. However, nonlinearities can alter original data structures, forming new, complex ones even for simpler linear relations.

In recommendation systems and fraud detection, data comprises distinct cohorts as people have different characteristics, but individuals with similar traits form a simple structure. 
Collaborative filtering uses a user-item matrix and matches users with similar interests and preferences by calculating similarity~\citep{su_survey_2009}, deep learning-based recommendation systems learn complex non-linear relations between users and items and make recommendations~\citep{he_neural_2017},~\citep{elkahky_multi-view_2015},~\citep{zhang_deep_2019}.
However, existing methods focus on finding the most similar instances to a query versus which instances form simple structures. 

Our approach of approximating the underlying data distribution through an ensemble of simple structures shares conceptual parallels with Gaussian Mixture Models (GMMs) and clustering. GMMs model data as a combination of Gaussian distributions, probabilistically assigning data points to clusters using the Expectation-Maximization (EM) algorithm \citep{bishop_pattern_2006}. However, GMMs aim to fit the entire dataset to a predefined number of distributions, often missing meaningful local structures and performing poorly in real-world applications compared to advanced models like XGBoost and Neural Networks \citep{wan_novel_nodate}.

In contrast, our method assumes datasets comprise multiple, potentially overlapping distributions and focuses on identifying simple structures—subsets with locally significant, linearly separable classes. Unlike GMMs, which provide a global clustering solution, our approach uncovers interpretable local patterns that highlight distinct subgroups, offering deeper insights into data heterogeneity.

Unlike traditional top-down methods that focus on minimizing global loss or boosting approaches that iteratively fit closer to the data without identifying inherent structures, our bottom-up method explicitly searches for simple structures. Instance-based learning optimizes data to reduce global loss, recommendation systems identify similar instances, and GMMs cluster data in feature space; however, none of these approaches explicitly address the separability of classes in local regions, which is central to the concept of simple structures. By focusing on class separability within localized regions, our method offers a distinct and interpretable approach to modeling data. This is particularly relevant in real-world scenarios, such as medical data, where identifying distinct distributions can support precision healthcare by tailoring decisions to specific populations \citep{colijn_toward_2017}.

\section{Methods}\label{sec4}

In this section, we introduce the simple structure identifying algorithm in detail. First, we present the assumptions for the data with simple structures. The main theorem of the paper shows that given assumption 1 and assumption 2, the algorithm will identify all the simple structures present in the dataset without any loss. To scale the algorithm to practical datasets with deviations from the assumptions, we describe the proposed heuristics.

\subsection{Background}

A simple structure refers to a subset of instances within a dataset that exhibits straightforward, localized patterns, despite the overall dataset being complex due to intricate feature-target interactions. These structures are characterized by regions where the complexity of interactions is significantly reduced, making the patterns easier to model and interpret.

Within each simple structure, the data can be further divided into linearly separable subsets, where each subset corresponds to a unique class. In this hierarchy, the simple structure represents a subset of the entire dataset, while the linearly separable subsets represent divisions within the simple structure based on class separability. 
\textbf{ In short, theoretically, a dataset consists of multiple simple structures, each simple structure made up of linearly separable subsets of data instances, where each subset within a simple structure corresponds to a unique class.}

To illustrate, in Figure \ref{fig:SS_Boundry}, the dataset contains two simple structures. In the region of Simple Structure 1, there are two linearly separable subsets corresponding to Class 1 and Class 2. Similarly, in the region of Simple Structure 2, there are three linearly separable subsets corresponding to Class 2, Class 3, and Class 1. This demonstrates how each simple structure contains its own linearly separable subsets, emphasizing the localized simplicity within these regions.

First, we will define few variables and custom functions then state the assumptions made in the simple structure dataset which will be leveraged to identify them -  \newline

In this paper, we denote the dataset as \( D \), which consists of \( n \) instances, represented as:
\[
D = \{x_1, x_2, \dots, x_n\}.
\]

We hypothesize that \( D \) can be decomposed into \( h \) disjoint \textit{simple structures}, denoted by \( S_1, S_2, \dots, S_h \) such that:
\[
D = \bigcup_{j=1}^{h} S_j, \quad \text{and} \quad S_i \cap S_j = \emptyset \; \text{for} \; i \neq j,
\]

where \( S_j \) represents the \( j \)-th simple structure. 
Each simple structure \( S_j \) is a subset of \( D \) containing instances that exhibit localized, simplified patterns, which can be modeled independently.  
The condition \( S_i \cap S_j = \emptyset \) ensures that no instance belongs to more than one simple structure. 

Each simple structure \( S_j \) consists of \( i \) subsets of instances, where each subset corresponds to a unique class. This can be expressed as:

\[
S_j = \bigcup_{i=1}^{m_j} S_j^i,
\]

where \( S_j^i \) represents the \( i \)-th subset within the \( j \)-th simple structure, and \( m_j \) is the total number of unique classes within \( S_j \). Each subset \( S_j^i \) is defined by instances belonging exclusively to a single class.

A key property of these subsets is \textit{linear separability}. Within a given simple structure \( S_j \), it is always possible to separate any subset \( S_j^i \) from another subset \( S_j^k \) (\( k \neq i \)) using a linear decision boundary. This ensures that the interactions between features are simplified within each simple structure, facilitating the use of interpretable models like linear classifiers for decision-making.

For instance, if \( S_j \) represents a simple structure, and \( S_j^1 \) and \( S_j^2 \) are its subsets, then there exists a hyperplane that can separate instances in \( S_j^1 \) from those in \( S_j^2 \). This property of linear separability highlights the localized simplicity of patterns within a simple structure, even if the overall dataset \( D \) exhibits complex feature-target interactions.

Let \( N(.) \) denote the function that provides the neighborhood of an instance. Further, let \( c \) represent the class of an instance. The set of instances belonging to class \( c \) within the neighborhood of an instance in the dataset \( D \) is defined as \( B_{c}^{D}(.) \).

\[
B_{c}^{D}(x) = \{x' \in N(x) \;|\; \text{class}(x') = c\},
\]

where \( x \) is the reference instance in \( D \), and \( N(x) \) represents the neighborhood of \( x \).

\subsection{Assumptions for Simple Structure}\label{subsec1}

By simple structures, we mean that the following two assumptions are satisfied for each simple structure \( S_j \). Specifically, we assume that all simple structures are separated from each other. This is formally defined as follows:

\begin{assumption}[Assumption of Separation]
Let \( x \in S_j \), where \( j = 1, 2, \dots, h \). Then the following conditions hold:
\[
\forall x \in S_j, \; N(x) \cap S_k = \emptyset \quad \text{and} \quad \forall x \in S_j, \; x \notin N(x') \; \text{where} \; x' \in S_k, \; k \neq j.
\]
\end{assumption}

The first condition ensures that the neighbors of any instance \( x \) in \( S_j \) do not belong to a different simple structure \( S_k \). The second condition ensures that no instance \( x \) in \( S_j \) is a neighbor of an instance \( x' \) in \( S_k \), where \( k \neq j \).

The second assumption assumes the absence of isolated subsets within a simple structure, stipulating that each instance must have a neighbor in one of the adjacent subsets within the simple structure. This assumption also assures that no subset of a given simple structure remains undetected. 

As stated previously, \( S_j^i \) represents the \( i \)-th subset within the \( j \)-th simple structure. We now define a \textit{transition instance} \( x_t \) from \( S_j^{i-1} \) as an instance such that all or some part of the neighborhood of \( x_t \), denoted by \( N(x_t) \), lies in \( S_j^i \). Formally:

\[
x_t \in S_j^{i-1}, \quad \text{and} \quad N(x_t) \cap S_j^i \neq \emptyset.
\]

Using this definition, we call \( S_j^{i-1} \) the \textit{predecessor} of \( S_j^i \), as there exists a transition instance \( x_t \) linking them. Furthermore, if \( S_j^i \) also contains a transition instance \( x_t' \) such that its neighborhood overlaps with \( S_j^{i-1} \), i.e.,

\[
x_t' \in S_j^i, \quad \text{and} \quad N(x_t') \cap S_j^{i-1} \neq \emptyset,
\]

then \( S_j^{i-1} \) and \( S_j^i \) can be considered \textit{predecessors of each other}.

\begin{assumption}[Assumption of Transition]
The following three conditions must be met to ensure the algorithm can always progress:

(a) If a simple structure \( S_j \) is composed of \( m \) linearly separable subsets \( S_j^i \), where each subset corresponds to a unique class, and \( S_j^i \) transitions from its predecessor \( S_j^{i-1} \), then there exist at least \( m-1 \) transition instances \( x^t \) connecting \( m-1 \) unique pairs of \( S_j^i \) and \( S_j^{i-1} \). Formally:
\[
\forall x^t \in S_j^{i-1}, \; \exists x' \in N(x^t) \; \text{such that} \; x' \in S_j^i,
\]
where \( N(x^t) \) is the neighborhood of \( x^t \).

(b) An isolated subset does not exist within the simple structure. That is, for any subset \( S_j^i \subset S_j \),
\[
\nexists S_j^i \; \text{such that} \; \forall x \in S_j^i, \; N(x) \subseteq S_j^i \; \text{and} \; \forall x \in N(x'), \; x' \in S_j^i.
\]
This ensures that each subset \( S_j^i \) is connected to other subsets in \( S_j \).

(c) There are no outliers in a subset of a simple structure. Formally:
\[
\nexists x \in S_j^i \; \text{such that} \; \nexists x' \in S_j^i \setminus \{x\} \; \text{where} \; x \in N(x').
\]
This guarantees that every instance in \( S_j^i \) has a connection to at least one other instance within the same subset.
\end{assumption}

\subsection{Vanilla Simple Structure Identifying Algorithm}\label{subsec2}

We approximate each of the  $h$  simple structures in the dataset using subsets of instances. To identify these simple structures, we leverage both Assumption 1 (Separation) and Assumption 2 (Transition). According to Assumption 2, instances or subsets of instances within a simple structure are connected through their neighborhoods, while Assumption 1 ensures there are no connections between distinct simple structures.

To operationalize these assumptions, we employ the  K -Nearest Neighbors (KNN) space computed using Euclidean distance to identify simple local structures. KNN serves as a proximity measure, capturing local relationships within the data. The nearest neighbors of an instance are expected to share similar features responsible for their class, thereby forming a simple structure.

Since the size and characterization of instances in a simple structure are unknown, the algorithm begins with a random seed instance and recursively expands to include its nearest neighbors. This recursive expansion ensures that all connected instances belonging to the same simple structure are included. The value of  K  in the KNN space is chosen to satisfy both Assumption 1 and Assumption 2. Thus now \( N(.) \) denotes the function that provides the K nearest neighbors of an instance and  \( B_{c}^{D}(.) \)  represents the set of instances belonging to class \( c \) within the K nearest neighbors of an instance in the dataset \( D \).

Suppose we start with data point $x\in S_j$. We want to generate a set $R_j(x)$ using seed $x$ such that $R_j(x) = S_j$ for some K satisfying Assumption 1 and Assumption 2.  The first step is  
{\setlength{\abovedisplayskip}{1pt}
\setlength{\belowdisplayskip}{1pt}
\begin{equation} \label{step1}
    R_j^{1}(x)= x \cup N(x)
\end{equation}}
After $k$ steps we have a subset $R_j^{k}(x)$ and we have newly added points 
{\setlength{\abovedisplayskip}{1pt}
\setlength{\belowdisplayskip}{1pt}
\begin{equation} \label{step2}
    A(k)=R_j^{k}(x)\setminus R_j^{k-1}(x)
\end{equation}}
The simple structure identifying algorithm will only grow on the data points in $A(k)$; that is, in the next iteration, we will add the data points.
{\setlength{\abovedisplayskip}{1pt}
\setlength{\belowdisplayskip}{1pt}
\begin{equation} \label{step3}
    \tilde A(k+1) = \bigcup_{x\in A(k)} N(x)
\end{equation}}
However, some of those data points may already be in $R_j^{k}(x)$. Thus,
{\setlength{\abovedisplayskip}{1pt}
\setlength{\belowdisplayskip}{1pt}
\begin{equation} \label{step4}
    A(k+1) = \tilde A(k+1) \setminus R_j^{k}(x)
\end{equation}}

The growth for the simple structure $S_j$ will stop when at a step $k$ the $A(k) = \emptyset$. 
 This recursion growth will include all the instances and all the $S_{j}^{k}$ subsets in an identified simple structure as long as Assumption 1 and Assumption 2 are satisfied.  Recursive growth will stop automatically as the data satisfies Assumption 1 blue{with selected K}.
 Simple structure search will stop automatically when we have identified all the simple structures or every instance is assigned to some simple structure. 
The main theorem of the paper shows that given assumptions 1-2, the above steps will identify all the simple structures present in the dataset without any loss. 
For simplicity, we constrain ourselves to a data set $D$ with two simple structures $S_1$ and $S_2$ such that $D=S_1 \cup S_2$ and $S_1 \cap S_2 = \emptyset$.

\begin{theorem}
Let $S_1$ be a simple structure composed of two linearly separable subsets $S_{1}^1$ and $S_{1}^2$. Let $R_{1}(x)$ denote the subset constructed from a starting point $x\in S_1$ using the simple structure identifying algorithm. If Assumption 1 and Assumption 2 hold, then vanilla simple structure identifying algorithm would yield
$R_{1}(x)=S_1$
In other words, the vanilla simple structure identifying algorithm will discover the simple structure corresponding to its starting point.
\end{theorem}

\begin{proof}
(i) We start by showing that $R_{1}(x)\subseteq S_1$ by contradiction. Assume that it is not true. Then there exist an $x' \in R_{1}(x)$ such that $x'\notin S_1$. Let $i'$ denote the first time that such a point is constructed, namely
    \[
    i' = \min\{i:A(i)\cap S_2 \neq \emptyset\}
    \]
    \indent Let $x'\in A(i')\cap S_2$. Then $x'\in N(x'')$ for some $x''\in R_1^{i'-1}$ and $x''\in S_1$ by the definition of $i'$, which contradicts Assumption 1.  This shows that $R_{1}(x)\subseteq S_1 \label{r_in_s}$.
    
    (ii) Now we show that $S_1 \subseteq R_{1}(x)$, again by contradiction. Assume that it is not true. Then assume there exists an $x' \in S_1$ such that $x' \notin R_{1}(x)$. 
    
    Let $x$ denote the seed of the algorithm and assume $x \in S_{1}^1$. First, consider the case where $x' \in S_{1}^2$. As $x$ is the seed we can state that $x\in R_{1}(x)$, and by Assumption 2b, $S_{1}^1 \in  R_{1}(x)$. Now consider the case of $x' \notin R_{1}(x)$ in two parts. 
    
    a) First, we know $x' \in S_{1}^2$ and assume $\exists  x'' \in S_{1}^2$. We will only have
    $x' \notin R_{1}(x)$ iff $\nexists  x^t \in S_{1}^{1} : x'' \in N(x^t)$. 
    In other words $x' \in S_{1}^2$ but there is no transition instance $x^{t} \in S_{1}^1$. Without transition instance it would be impossible to grow from $S_{1}^{1}$ and $S_{1}^{2}$. This contradicts Assumption 2a and there must be at least one $x^t$ between $S_{1}^{1}$ and $S_{1}^{2}$. Hence it is not possible that $\nexists  x^t \in S_{1}^{1} : x'' \in N(x^t)$. 
    
    b) Second, we know $x' \in S_{1}^2$, and $\nexists x'' \in S_{1}^2 \setminus \{x'\} : x' \in N(x'')$. In other words,  $x' \in S_{1}^2$, but it cannot be reached by any of the nearest neighbors from $S_{1}^{2}$ as it forms an isolated cluster within $S_{1}^2$. But this contradicts Assumption 2c, so this case is also false. 
    
    As both cases lead to a contradiction, the Assumption that there exists an $x' \in S_1$ such that $x' \notin R_{1}(x)$ is false, and we must have $S_1 \subseteq R_{j}(x)$. Thus, $R_{1}(x) = S_{1}$ and the theorem is proven.
     
\end{proof}

The above theorem fails if $R_{1}(x)\neq S_1$ using the algorithm. This happens either by violating Assumption 1 and growing into the adjacent simple structure $S_2$ or by excluding a part of $S_1$ by violating Assumption 2. 

\begin{algorithm}[H]
    \caption{Vanilla Simple Structure Identifying Algorithm}\label{alg:vanilla_algo}
    \begin{algorithmic}
    \Require KNN K, data set D
    \State initialization $S = \{\}$ \Comment{List of identified simple structure}
    \While {$D \neq \emptyset$}
     \State initialization $R_j = \emptyset$, 
     \State initialization $Ak = \emptyset$ \Comment{Set for neighbours to be explored}
     \State Sample  $x \sim D$ 
     \State $R_j \gets x$
     \State $Ak \gets N(x)$
     \While {$Ak \neq \emptyset$}
     \State Sample $i\sim AK$
     \State $R_j \gets R_j \cup i$
     \State $Ak \gets Ak \cup (N(i)\setminus R_j)$
     \State $AK \gets AK \setminus i$
     \EndWhile
     \State $D \gets D \setminus R_j$
     \State $S \gets S \cup \{R_j \}$ \Comment{Append identified simple structure to a list of simple structure}
     \EndWhile
     \State return $S$
    \end{algorithmic}
\end{algorithm}

\subsection{Algorithm with Growth Rule}\label{subsec3}

In real life, satisfying Assumption 1 in real datasets is hard.
Assumptions 1 and 2 are restrictive but are only needed to guarantee that the algorithm always discovers the exact desired structure. It may still obtain a good approximation despite deviations from one or both assumptions.
For assumptions 2b and 2c if such isolated subsets or outliers exist, then ignoring those is an acceptable outcome. When Assumption 1 (Assumption of separation) is violated, then the $R_j(x) \cap S_k \neq \emptyset$, that is, distributions that comprise the data or simple structures overlap, and we need to remedy this.
In the edited nearest neighbor algorithm, the algorithm assumes different classes come from different distributions, and repeated removal of misclassified KNN instances makes such distribution disjoint~\citep{noauthor_experiment_1976}. 
Thus, we use a heuristic to stop the growth of the $R_j(x)$ into $S_k$ and modify the vanilla algorithm by adding the rule of growing simple structure only on the correctly classified instances by KNN algorithm for a $K$.

Similar to vanilla simple structure identifying algorithm after $k$ steps, we have a subset $R_j^{k}(x)$ and we have newly added points  according to step (\ref{step2}) 
$A(k)=R_j^{k}(x)\setminus R_j^{k-1}(x)$.
But now these can be divided into two disjoint subsets $A(k)=A_1(k)\cup A_2(k)$, where the points in $A_1(k)$ are correctly classified by  KNN using $D$, and $A_2(k)$ are incorrectly classified. The simple structure identifying algorithm will only grow on the data points in $A_1(k)$; that is, we will modify step (\ref{step3}) as
$\tilde A(k+1) = \bigcup_{x\in A_1(k)} N(x)$. Although some of those data points may already be in $R_i^{k}(x)$.

The following lemma and theorem show why the above heuristic might help to handle the violation of Assumption 1. Lemma 1 shows when complete failure of the heuristic and algorithm is possible.

\begin{lemma}
If Assumption 1 (Separation) is violated, the algorithm will expand into a distinct simple structure \( S_2 \) if there exists at least one correctly classified instance \( x' \in S_1 \) such that:
\[
N(x') \cap S_2 \neq \emptyset \quad \text{and} \quad |B_l^D(x')| \geq \left\lceil \frac{k}{2} \right\rceil.
\]
Under these conditions, the region grown from \( x \), denoted as \( R_1(x) \), will include instances from both \( S_1 \) and \( S_2 \), resulting in:
\[
R_1(x) \subseteq S_1 \cup S_2.
\]
\end{lemma}

\begin{proof}
Proof by contradiction. Let's assume when the Assumption of separation is violated, the simple structure identifying algorithm will not grow into $S_{2}$.

Then $ R_{1}(x) \subseteq S_1 $ will be true. We know there exits  $x'$ such that $x' \in S_1 $ and $ N(x') \cap S_2 \neq \emptyset$ and $|B_{l}^D(x')| \geq \lceil \frac{k}{2} \rceil$.

From Assumption 2a, 2b and 2c  let $i'$ be the first step when $x' \in A_{1}(i')$ and by definition  $A_{1}(i') \in R_{1}(x)$ thus $x' \in R_{1}(x)$. 

Let $x'' \in N(x') \cap S_2$ and assume $|B_{l}^D(x'')| \geq \lceil \frac{k}{2} \rceil$. And thus by definition $x'' \in A_1(i'+1)$ and $A_1(i'+1) \in R_{1}(x)$ so $x'' \in R_1(x)$. As  $x'' \in  S_2$ and  $x'' \in R_1(x)$ is true then $R_1(x) \subseteq S_1$ is not true. Hence our Assumption is wrong. 

In other words, when the Assumption of separation is violated, and there exists at least one correctly classified  instance $x'$ in $S_{1}$ and at least one correctly classified neighbor of  $x'$ in $S_{2}$, we grow into simple structure $S_2$

\end{proof}

The following theorem provides insights into under what conditions the algorithm may succeed for sure, even if Assumption 1 is not satisfied.

\begin{theorem}
If Assumption 1 (Separation) does not hold and the simple structure identifying algorithm begins with a seed instance \( x \in S_1 \), the algorithm will \emph{not} expand into \( S_2 \) if the following conditions are satisfied:  
\[
\forall x' \in S_1, \; N(x') \cap S_2 \neq \emptyset \quad \text{and} \quad |B_l^D(x')| < \left\lceil \frac{k}{2} \right\rceil,
\]
That is, all instances \( x' \in S_1 \) whose neighborhoods overlap with \( S_2 \) must be misclassified. Under these conditions, the region grown from \( x \), denoted as \( R_1(x) \), will remain contained within \( S_1 \), formally:
\[
R_1(x) \subset S_1.
\]
\end{theorem}

\begin{proof}
Let's assume the Assumption of separation is violated at one instance only, i.e., $x'$.

We have $x' : x' \in S_1 $ and $ N(x') \cap S_2 \neq \emptyset$ and $|B_{l}^D(x')| < \lceil \frac{k}{2} \rceil$. 

But as $x'$ is misclassified, we won't grow on $x'$; hence, it doesn't matter if $ N(x') \cap S_2 \neq \emptyset$ is true or not. As we won't grow on $x'$, $x' \notin R_{1}(x)$  even though $x' \in S_1$. 

    Hence $R_{1}(x) \subset S_1$. Thus, the heuristics of not growing on misclassified instances should result in a good approximation of $S_1$ in this case.
    
\end{proof}

Theorem 2 is intuitive as if all the instances violating Assumption 1 are misclassified, then we will ignore them; hence, the identified simple structure would be a subset of the original one. 

When the Assumption of separation is violated, two different simple structures with different majority classes overlap. It is thus unlikely that none of the conditions in Lemma 1 or Theorem 2 will be satisfied. But we can assume that many instances in the set $S_1 \cap S_2$ will be miss-classified because of the unstructured overlapping region. Thus growth rule, along with Assumption 2 violation in practical data,  will stop the significant growth of a simple structure $S_j$ into its neighboring simple structure $S_k$. Thus, our algorithm will yield a good approximation of the underlying structure.

\subsection{More Heuristics for Practical Implementation}\label{subsec4}

Violations of assumptions 2b and 2c are common in practice and can impact identified structures. This, coupled with uneven data density, can cause sensitivity of the algorithm to the starting seed.
For example, if the starting seed is within the instances violating assumptions 2b or 2c, a specific and difficult-to-generalize structure may result. Conversely, violating Assumption 1 can cause some structures to become larger and less specific than the actual simple structure, which should be avoided. In reality, the actual simple structure will fall between these two extremes.
As is done for other randomized algorithms that may be sensitive to the initial seeds (e.g., \citep{na_research_2010}),  we use a multi-start seed approach and introduce a new tunable regularization parameter called SST (simple structure size threshold) to avoid undesired simple structures. 

Now $P$ independent seeds produce $P$ simple structures denoted by $R_{p} : p \in [1, \dots, P]$ .
If $e_{p}$ is the KNN error of a simple structure,  then any simple structure with size away from SST in either direction is penalized linearly, and regularized error is given by $\beta_{p}$ in \ref{regulerization}. Due to multiple starting seeds, we will get various structures with various sizes and 
SST helps avoid extremely small and large simple structures and aids in choosing a simple structure with an acceptable size from multi-start iteration.  
{\setlength{\abovedisplayskip}{1pt}
\setlength{\belowdisplayskip}{1pt}
\begin{equation} \label{regulerization}
    \beta_{p} = 
    \begin{cases}
        e_{p} + (\alpha_{l} \times (SST - |R_{p}|)) & |R_{p}| \leq SST \\
        e_{p} + (\alpha_{u} \times (|R_{p}| - SST)) & |R_{p}| > SST
    \end{cases}
\end{equation}
 In \ref{regulerization}, the $\alpha_{l}$  and $\alpha_{u}$ scale the penalty to the magnitude of the $e_{p}$ and control preference of the size of the simple structure. 
 When there is no to moderate deviation from the theoretical assumptions, then the algorithm is less sensitive to SST. The size of deviated simple structures will be in the extreme direction of SST, and the true simple structures will lie in between. Thus, a good approximation will always be found due to regularization; and even if a simple structure size differs from the given SST value, it will be found in a later iteration of the algorithm.
 However, large deviations from the assumptions may cause the deviated simple structure's size to vary significantly and the chosen structure will be sensitive to SST. To tune SST in such cases, we can evaluate the overall performance of all the simple structures for various SSTs, and examine the underlying structures with domain knowledge.

In practice, further iterations of subset searching for simple structures often result in very small sizes made up of unassigned instances from Assumption 2b and 2c violation. Therefore, we stop searching for simple structures after a predetermined percentage of the data is assigned to simple structures and assign unassigned instances to the nearest major simple structure centroid. 
The algorithm assumes a single value of K is appropriate for all simple structures, but the later structures may have less similarity than the earlier ones, as they are composed of rejected instances. To account for this, we recommend adapting the value of K by incrementing it in later iterations of the algorithm based on the specific dataset. 

The following pseudocode  \ref{alg:sub_algo_hur}  shows the simple structure identifying algorithm along with all the heuristics stated in the main paper. The algorithm uses heuristics of not growing recursively on the misclassified instances, multi-seed start, along with regularization and centroid allocation.

\begin{algorithm}[H]
    \caption{ Simple Structure Identifying Algorithm with Heuristics}\label{alg:sub_algo_hur}
    \begin{algorithmic}
    \Require KNN K, data set D, SST, $\alpha_{l}$, $\alpha_{u}$
    \State initialization $S = \{\}$ 
    \While {Data addressed $\leq$ 90\%}
     \State P = $\lceil 0.10 \times |D| \rceil$
     \State $Seed = \emptyset$
     \For {$s \in \{1,2,3, \cdots,P \}$} 
     \State sample $x_{s} \sim D \setminus Seed$ 
     \State $Seed \gets  Seed \cup x_s$
     \EndFor
     \State initialization $R_{j} = \emptyset$, 
     \For{$x_{p} \in Seed$} 
     \State initialization $R_{p} = \emptyset$,  $AK = \emptyset$ 
     \State $R_{p} \gets x_{p}$, $AK \gets N(x_{p})$
     \While {$AK \neq \emptyset$}
     \State Sample $i\sim AK$
     \If{$|B_{l}^{D}(i)| \geq \frac{K}{2}$}
     \State $R_{p} \gets R_{p} \cup i$
     \State $AK \gets AK \cup (N(i)\setminus R_{p})$
     \State $AK \gets AK \setminus i$
     \Else
     \State $AK \gets AK \setminus i$
     \EndIf
     \EndWhile
     \State $e_{p} \gets KNN(R_{p})$ 
     \If{$|R_{p}| \leq SST$} 
     \State $\beta_{p} \gets e_{p} + (\alpha_{l} \times (SST - |R_{p}|))$
     \Else
     \State $ \beta_{p} \gets e_{p} + (\alpha_{u} \times (|R_{p}| - SST))$
     \EndIf
     \EndFor
     \State $R_{j} \gets R_{p} : p \in argmin\{\beta_{1}, \beta_{2}, \cdots \beta_{p} \} $
     \State $D \gets D \setminus R_j$
     \State $S \gets S \cup \{R_j \}$ 
     \If{data addressed $\geq$ 0.5}
     \State Adapt K
     \EndIf
     \EndWhile
     \For{$x_{i} \in D$}
     \State $x_{i} \in R_{j} : j = argmin(dist(x_{i},C_{j}))$ 
     \EndFor
     \State return $S$
    \end{algorithmic}
\end{algorithm}

\section{Results}\label{sec2}

In this section, we report the algorithm's effectiveness in identifying simple structures within datasets, along with an in-depth comparison with Gaussian Mixture Models (GMMs). Through our analysis, we aim to demonstrate the algorithm's capacity to enhance predictive performance and model reasoning. Additionally, we investigate the robustness of our approach against deviations from Assumption 1, utilizing analytical as well as practical examples to provide insights into its reliability across various scenarios.

\subsection{Algorithm Robustness}

We generate synthetic data as shown in figure \ref{fig:0_percent_structure} in two dimensions for algorithm testing. The data consists of two independent normal distributions with variances of 4 and means of 1 and 15, creating a simple structure ($S_1$) with 286 data points and another simple structure ($S_2$) with 283 data points. Assumption 1 holds for the data. We relax the precise satisfaction of assumptions 2b and 2c.

Assumption 1 is often violated in practical data. To address this, we use the heuristic of no growth on misclassified instances. We evaluate the robustness of this approach by measuring the overlap between $S_1$ and $S_2$, by counting the percentage of $S_1$ instances with at least one of the $K$ nearest neighbors in $S_2$. We assess the predictability of simple structures by building decision trees on them.

Given the sensitivity of estimates to seed and Assumption violations, we assess estimated variability using bootstrapping. We generated 20 bootstrap samples for a scenario where $S_1$ is pushed towards $S_2$. Our approach utilizes an initial $K=6$, SST$=250$, $\alpha_l = 0.125$, and $\alpha_u = 0.3$, with 20\% data used as seed. We increase $K$ to 8 after addressing 40\% of the data.

Figure \ref{fig:overlap performance} displays the algorithm's predictive performance relative to deviations from Assumption 1. Each data point represents the mean bootstrap accuracy of simple structures for the mean percent bootstrap overlap. The accuracy of a single bootstrap sampling is a weighted average accuracy based on the corresponding simple structure's size. Shaded areas represent 95\% confidence intervals.

We compare decision trees learned on simple structures with those learned on the entire dataset using Bootstrap to estimate performance. We address handling unassigned instances in three ways: ignoring them, treating them as misclassified, or employing the centroid strategy. 
Here, for synthetic data, unassigned instances occur when the size of a simple structure is less than 50 or when some instances are left out after learning over 90\% of the data. Ignoring unassigned instances yields optimistic estimates, as some challenging instances are removed from simple structures (see figure \ref{fig:overlap performance}). Treating them as misclassified is a pessimistic estimate, as simple structures may classify many correctly. The centroid strategy provides fairer comparisons, consistently outperforming the whole dataset model. Hence, we solely report results using this strategy for subsequent analyses.

\begin{figure}[H]
     \centering
     \begin{subfigure}[b]{0.4\textwidth}
         \centering
         \includegraphics[width=\textwidth]{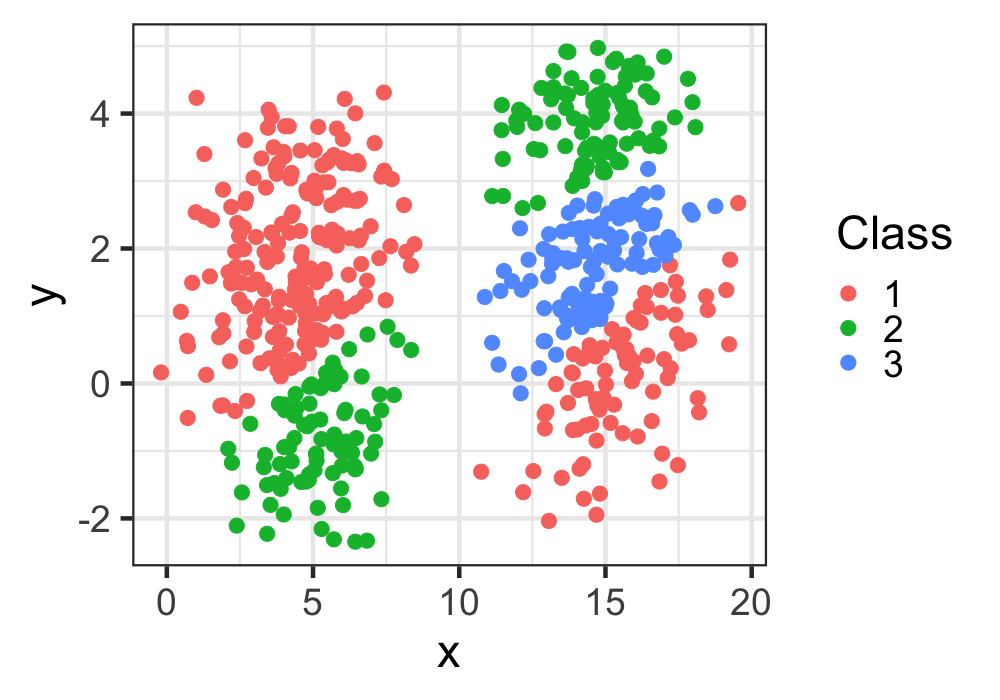}
         \caption{}
         \label{fig:0_percent_structure}
     \end{subfigure}
     \begin{subfigure}[b]{0.4\textwidth}
         \centering
         \includegraphics[width=\textwidth]{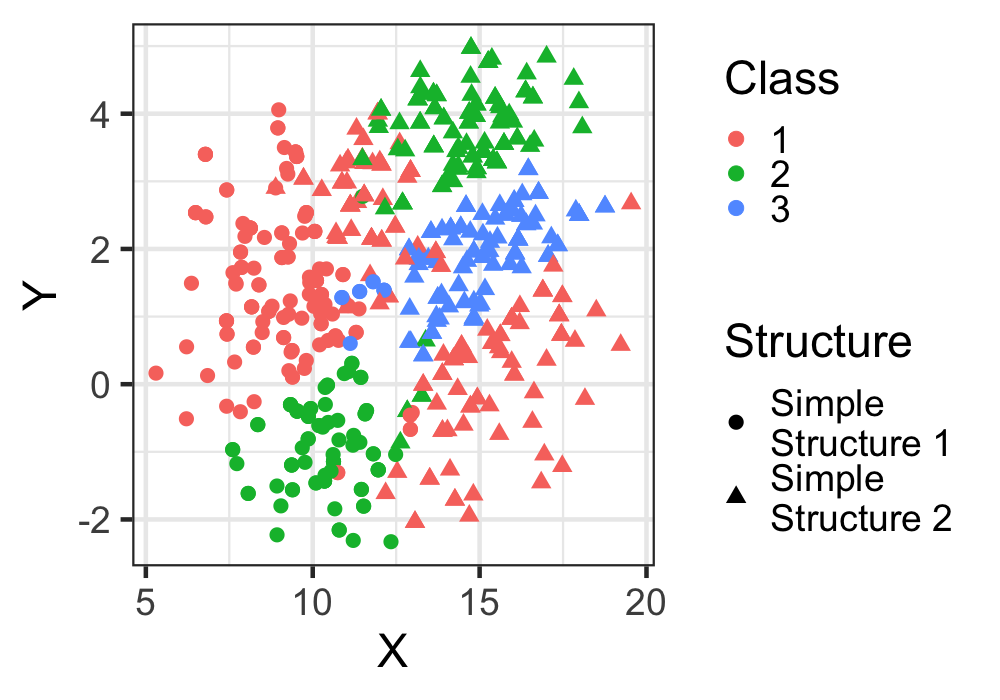}
         \caption{}
         \label{fig:28_percent_structure}
     \end{subfigure}
     \hfill
        \caption{(a) Parent synthetic data with two Gaussian distributions containing multiple classes: $S_1$ comprises subsets $S_{1}^{1}$ (majority class 1) and $S_{1}^{2}$ (majority class 2) linearly separated, while $S_2$ consists of subsets $S_{2}^{1}$ (class 2), $S_{2}^{2}$ (class 3), and $S_{2}^{3}$ (class 1)  linearly separated. (b) Identified simple structures using the simple structure identifying algorithm with a violation of Assumption 1, based on a bootstrap sample from (a).} 
        \label{fig:mid_overlap}
\end{figure}

Simple structure accuracy alone offers an incomplete view. To gain deeper insights, let $R_1$ and $R_2$ represent identified structures corresponding to $S_1$ and $S_2$. Precision, calculated as $\frac{|R_j \cap S_j|}{|R_j|}$, indicates the proportion of actual instances from the identified structure belonging to the ground truth. Recall, computed as $\frac{|R_j \cap S_j|}{|S_j|}$, informs us how many instances from the ground truth are captured by the identified structure.

\begin{figure}[H]
    \centering
    \includegraphics[width=0.75\textwidth]{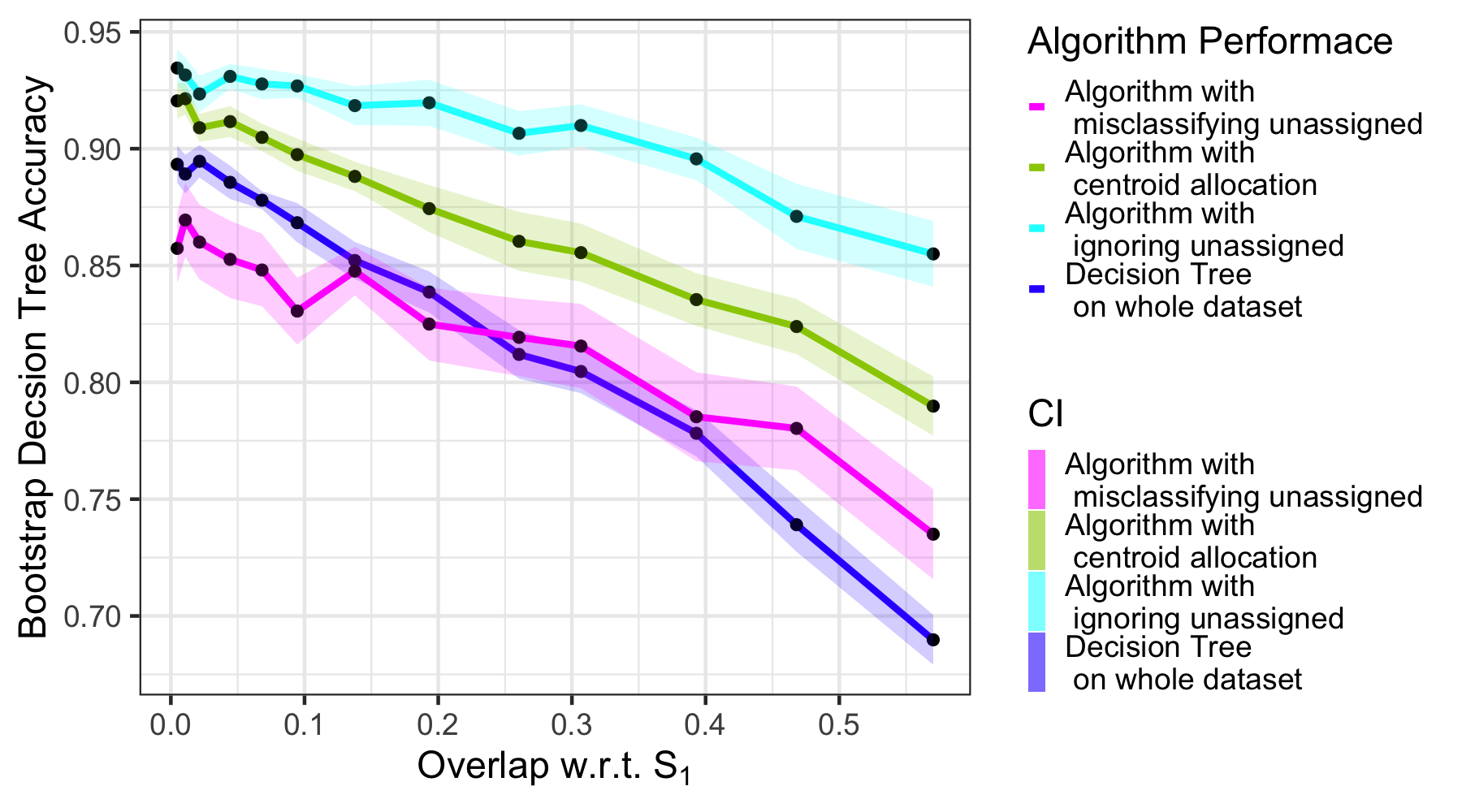}
    \caption{ 
    Comparison of bootstrap accuracy estimates for a single decision tree on the entire dataset versus an ensemble of decision trees on identified simple structures in synthetic data. Additionally, a comparison of various strategies for handling unassigned instances in the simple structure identifying algorithm. The x-axis denotes the violation of Assumption 1, represented as the percentage overlap of $S_1$.
    }
    \label{fig:overlap performance}
\end{figure}

Figure \ref{fig:precision_recall_comparison} displays bootstrap estimates of precision and recall with the same setup. For structures with less than 10\% overlap, precision, and recall typically exceed 85\%, never dropping below 80\%, indicating robustness to light to moderate deviation from Assumption 1. Higher overlap generally maintains mean values above 80\%, albeit with wider confidence intervals. Thus, we get approximations of simple structures resembling Figure \ref{fig:28_percent_structure} frequently. But for very high overlap $>40\%$ sharp fall in recall shows the complete failure of the algorithm.

In figure \ref{fig:precision_recall_comparison}, the consistent improvement in recall and a slight drop in precision after centroid allocation demonstrates its efficacy. Assigning unassigned instances to the nearest centroid mostly yields correct assignments, but occasional incorrect assignments lead to a slight precision decrease. For overlaps $>15\%$, precision drops notably due to significant Assumption 1 deviation. Overlapping adjacent simple structures affects centroid calculation, resulting in more frequent incorrect assignments.

\begin{figure}[H]
    \centering
    \includegraphics[width=0.75\textwidth]{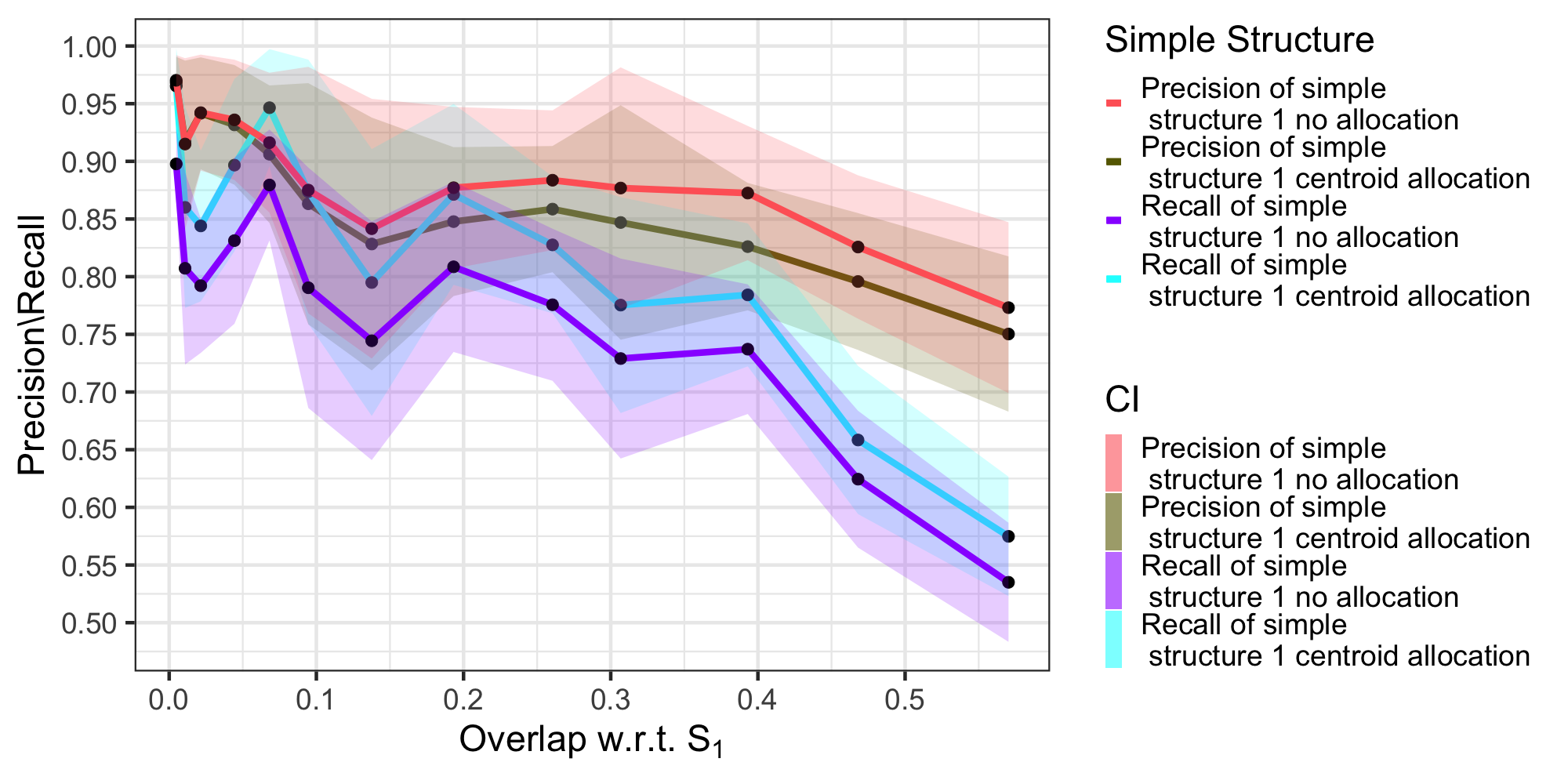}
    \caption{Assessment of the robustness of the simple structure identifying algorithm regarding violation of Assumption 1. The x-axis illustrates the degree of violation, denoted as the percentage overlap of $S_1$. Additionally, a comparison of bootstrap estimates of precision and recall for simple structure $S_1$ from synthetic data before and after centroid allocation in the simple structure identification algorithm.}
    \label{fig:precision_recall_comparison}
\end{figure}

\subsection{Comparison with Gaussian Mixture Models (GMMs)}

Identifying simple structures has parallels to well-known mixture models. 
GMMs are probabilistic models utilized to discern normally distributed sub-populations within larger population~\citep{wan_novel_2019}. However, our method is fundamentally different from clustering.  We claim that simple structures cannot be identified solely through clustering. In such instances, each component of the GMMs is perceived to represent either a single distribution or a simple structure. GMMs segment the data through clustering, with each cluster representing an underlying Gaussian distribution.

For such models, the parameters of the underlying distributions are typically estimated using the Expectation-Maximization (EM) algorithm~\citep{bishop_pattern_2006}. EM clustering serves as a natural benchmark for our method. However, our method diverges from GMMs by assuming that the data comprise simple structures composed of one or more classes separated from each other, without necessarily adhering to Gaussian distributions. Conversely, GMMs assume that data constitute a combination of multiple Gaussian distributions, and further, in GMM classification, it is assumed that a single Gaussian only corresponds to a single class~\citep{wan_novel_2019}. These assumptions may not hold true in real-life scenarios. Consequently, we demonstrate that our proposed simple structures cannot be effectively identified using GMMs.

To compare the simple structure identifying algorithm against GMM, we create different scenarios of the synthetic data as shown in the figure~\ref{fig: SS_vs_GMM }.
We create two data types for Gaussian and exponential distribution as shown in figure \ref{fig: SS_vs_GMM }

\begin{figure}[h]
    \centering
    \begin{subfigure}{0.45\textwidth}
        \centering
        \includegraphics[width=\textwidth]{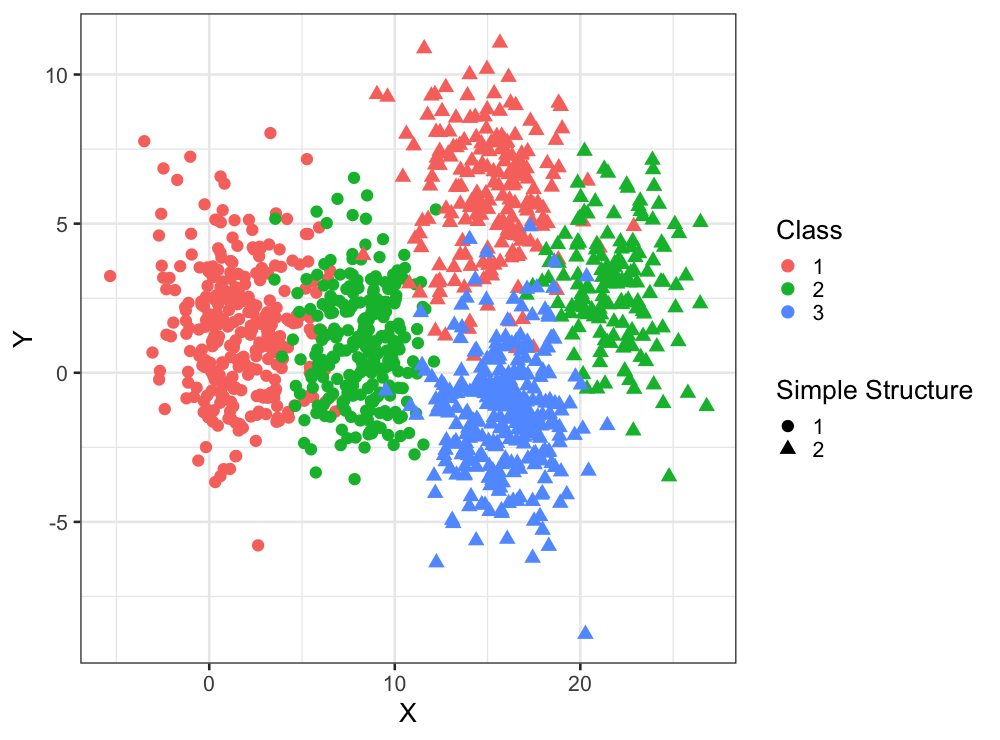}
        \caption{}
        \label{fig:scenario_1}
    \end{subfigure}
    \hfill
    \begin{subfigure}{0.45\textwidth}
        \centering
        \includegraphics[width=\textwidth]{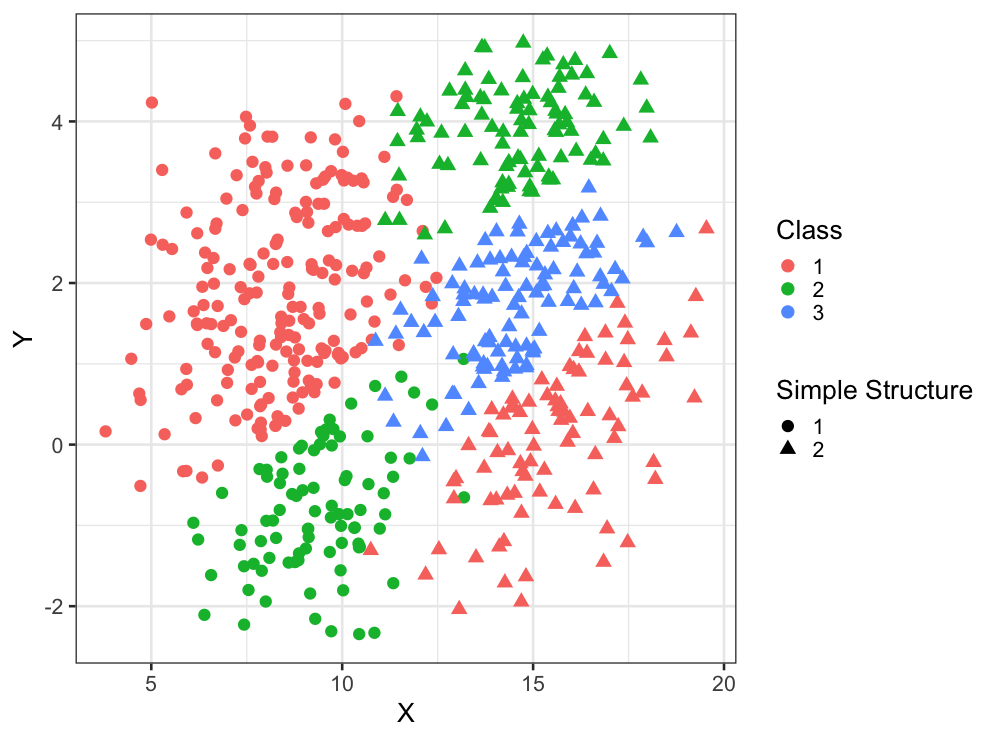}
        \caption{}
        \label{fig:scenario_2}
    \end{subfigure}
    
    \medskip
    
    \begin{subfigure}{0.45\textwidth}
        \centering
        \includegraphics[width=\textwidth]{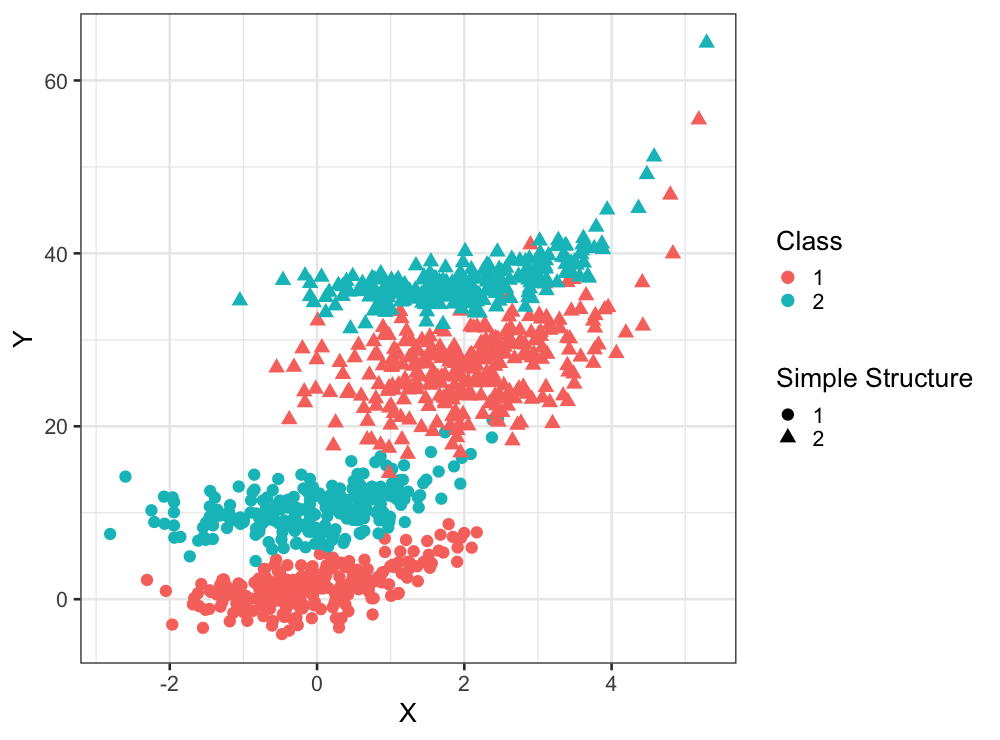}
        \caption{}
        \label{fig:scenario_3}
    \end{subfigure}
    \hfill
    \begin{subfigure}{0.45\textwidth}
        \centering
        \includegraphics[width=\textwidth]{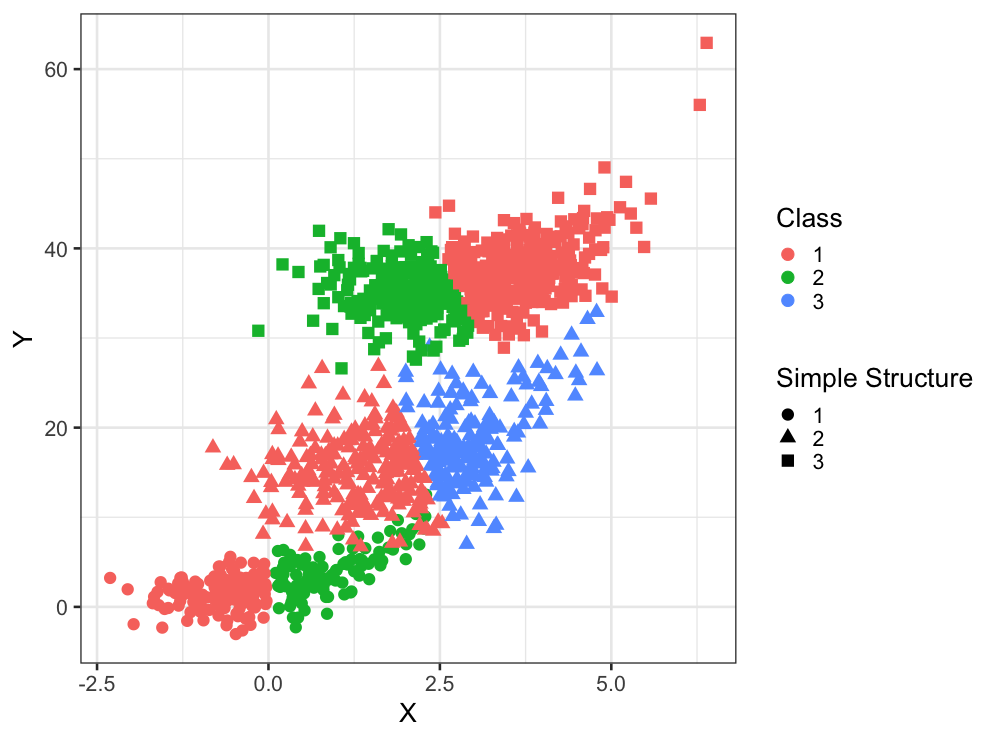}
        \caption{}
        \label{fig:scenario_4}
    \end{subfigure}
    \caption{Testing the robustness of the Simple Structure Algorithm and Gaussian Mixture models in identifying underlying simple structures across various data distributions. (a) Scenario with a Single Gaussian distribution containing a single class. (b) Scenario with a Single Gaussian distribution containing multiple classes. (c) Scenario with a Single Exponential distribution containing a single class. (d) Scenario with a Single Exponential distribution containing multiple classes.}
    \label{fig: SS_vs_GMM }
\end{figure}

We expect our algorithm to identify only one simple structure per simple structure of ground truth regardless of the number of distributions within it. GMMs approximate underlying distributions using EM, thus we expect it to identify single cluster for a given distribution regardless of the number of classes in it. The precision and recall of GMM and simple structure algorithm can be seen in table \ref{tab:gmm_ss_Precision_recall}.

In scenario 1, where all GMM assumptions hold, GMM achieves near-perfect identification, while our method struggles due to Assumption 1 violation. In scenario 2, where the Gaussian model is divided into multiple classes, GMM exhibits high precision but low recall, splitting structures rather than identifying them. 
Our method achieves moderate precision due to Assumption 1 violation, with comparable performance overall. 
In scenario 3, both methods have high precision, but our method shows slightly better recall. Moving away from Gaussian distributions, GMMs struggle, especially in scenario 4, where single exponential distributions are divided into multiple classes, resulting in poor recall. Due to GMMs' poor recall and structure splitting, post-processing steps are needed to identify underlying structures. In contrast, our method consistently achieves high recall, striving to recognize the underlying structures. Thus, our algorithm consistently outperforms GMMs in identifying underlying simple structures. Therefore, our method is robust and natural at identifying the underlying local simple structures.

\begin{table}[h]
\begin{tabular}{ccccccccc}
\hline
\multicolumn{2}{c}{\multirow{2}{*}{\textbf{Scenario}}}   &           & \multicolumn{2}{c}{\textbf{Structure 1}}                   & \multicolumn{2}{c}{\textbf{Structure 2}}                              & \multicolumn{2}{c}{\textbf{Structure 3}}      \\
\multicolumn{2}{c}{}                                                                 &           & \textbf{SS} & \textbf{GMM} & \textbf{SS} & \textbf{GMM} & \textbf{SS} & \textbf{GMM} \\ \hline
\multirow{4}{*}{\begin{tabular}[c]{@{}c@{}}\textbf{Gaussian} \\ \textbf{Distribution}\end{tabular}}    & \multirow{2}{*}{Single Class}   & Precision & 79.9   & 97.42  & 83.64   & 99.73    & NA   & NA   \\
&        & Recall    & 87.47   & 96.9  & 82.38   & 97.84    & NA   & NA    \\ \hhline{~--------}
& \multirow{2}{*}{Multiple Class} & Precision    & 72.79    & 95.85    & 72.41   & 95.99    & NA    & NA   \\  
&        & Recall    & 97.63     & 73.48         & 97.51     & 66.74   & NA      & NA   \\ \hline
\multirow{4}{*}{\begin{tabular}[c]{@{}c@{}}\textbf{Exponential} \\ \textbf{Distribution}\end{tabular}} & \multirow{2}{*}{Single Class}   & Precision & 90.57    & 97.27           & 97.54    & 95.67    & NA      & NA      \\
&        & Recall    & 85.82    & 81.34           & 79.80    & 78.88    & NA      & NA       \\ \hhline{~--------}
& \multirow{2}{*}{Multiple Class} & Precision     & 85.10    & 99.41    & 89.12   & 91.92    & 95.64           & 99.46     \\ 
&        & Recall    & 96.69    & 87.99  & 84.97   & 66.05   & 83.86   & 72.10     \\  \hline
\end{tabular}
\caption{The table compares the precision and recall of the simple structures (SS) identified by the Simple Structure Algorithm with those of the distributions identified by Gaussian Mixture Models (GMMs) across all scenarios depicted in Figure~\ref{fig: SS_vs_GMM }.}
\label{tab:gmm_ss_Precision_recall}
\end{table}

\subsection{Simple Structure Interpretability}

After the validation of the robustness of the algorithm to assumptions, we verify that simple structures learned by the models are interpretable.  We expect the models learned on the simple structures to be very simple and interpretable compared to a whole-data model.  Thus, we learn simple decision trees on the simple structures identified by the algorithm for data in figure \ref{fig:28_percent_structure}. Figure \ref{fig:simple_vs_whole} shows the decision tree boundaries of both models learned on the entire data and models on simple structures. 
The tree on the entire data has an accuracy of 81\% compared to 86\% for structure 1 and 82\% for structure 2 in the two discovered simple structures. 

More importantly, figure \ref{fig:simple_vs_whole} shows that decision boundaries learned on the entire data are more complex without better predictive ability than models on the simple structures. 
The root node of the decision tree on the entire dataset is $Y < 3.4$, which has no meaning from the original distribution's perspective. It splits the $S_2^{2}$ at an arbitrary location, and it is far-fetched to assume that class 1 is not present for any values $Y > 3.4$.
The decision boundaries from one part of the data create complex structures in the other part of the data. The boundary  $Y < 0.1$ and $x \geq 13$ classifies the $S_{1}^{2}$ correctly but splits the $S_{2}^{3}$ at arbitrary position. We thus need one more decision boundary $X > 12$ and $Y < 1.4$ to classify the instances in $S_{2}^{3}$ correctly, increasing the complexity of the model.
Even though the top-down model minimizes global loss but creates complex decision boundaries. Thus the decision boundaries learned by our ensemble models are much simpler than a global model.

\begin{figure}[H]
    \centering
    \includegraphics[width=0.7\textwidth]{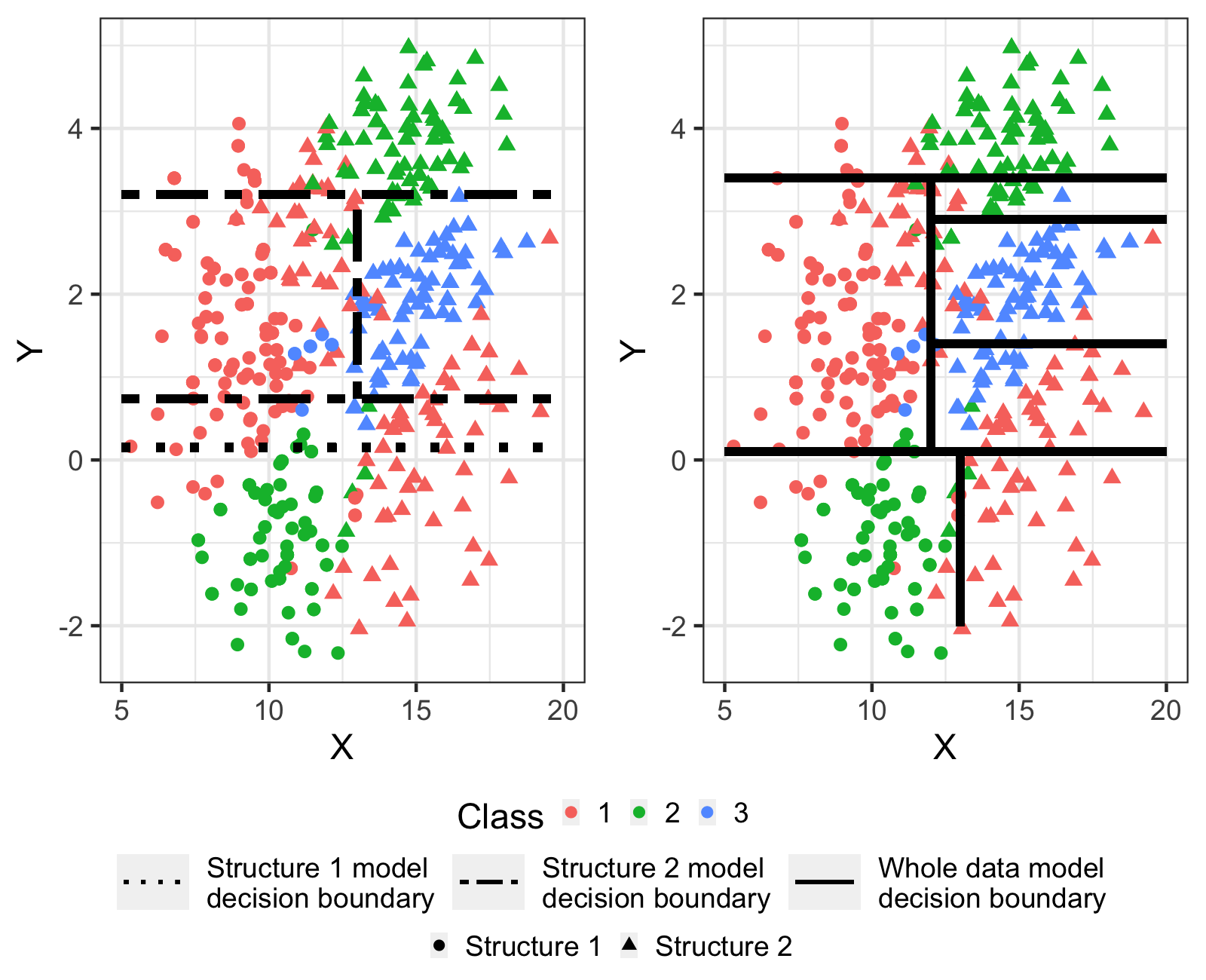}
    \caption{Comparison of decision boundaries learned by decision tree on the entire dataset (right) versus decision boundaries learned on the individual simple structures (left).}
    \label{fig:simple_vs_whole}
\end{figure}

\subsection{Simple Structure Performance on Real Datasets}

Next, we report how the algorithm uncovers meaningful simple structures and offers insights into real-life data. For this, we used popular open-source datasets from UCI and Kaggle, applying our Simple Structure (SS) algorithm to illustrate its performance and interpretability advantages.

Table~\ref{tab:UCI} summarizes the classification results on fourteen datasets of varying sizes, numbers of classes, and feature dimensions. We compare our SS method (using Logistic Regression built on each identified SS, denoted ``SS (LR)’’ in the table) to four established approaches:\emph{(i)} Logistic Regression (LR), using the \emph{entire} feature set,
\emph{(ii)} XGBoost, and
\emph{(iii)} Clustering-based LR, where we identify clusters in the data (to mimic Gaussian mixture models) and then fit LR on these clusters for comparison with our SS approach. In Table~\ref{tab:UCI}, we report the classification accuracy (with AUPRC in parentheses) using 5-fold cross-validation for all methods. 
For each fold, we employed bootstrap sampling on the training data to tune hyperparameters, using out-of-bag (OOB) accuracy to select the best configuration and then evaluated the resulting model on the held-out fold. 

In XGBoost, we searched over the number of trees (2 to 10), tree depth (2 to 10), learning rate (0.1 to 0.001), and both L1 and L2 regularization (0 to 5). For Logistic Regression, we selected among standard ranges of regularization strength and penalty type (e.g., L1 vs.\ L2). 
Because SS and the clustering-based LR rely on subgroup identification (i.e., simple structures or clusters), their final accuracy and AUPRC are averaged across these subgroups, weighted by the relative sizes of each subgroup to avoid biasing metrics toward particularly large or small clusters.
In both the clustering and SS identification algorithms, Gower distance was used for neighborhood calculations in mixed-type data~\cite{ahmad_survey_2019}. For clustering, we adopted the Partitioning Around Medoids (PAM)~\citep{Kaufman_pam_1990} algorithm, determining the optimal number of clusters via the silhouette metric.

\begin{table}[h]
\resizebox{\textwidth}{!}{%
\begin{tabular}{|l|c|c|c|c|c|c|c|c|}
\hline
Dataset                                                                                                                           & \begin{tabular}[c]{@{}c@{}}No. \\ Rows\end{tabular} & \begin{tabular}[c]{@{}c@{}}No. \\ Classes\end{tabular} & \begin{tabular}[c]{@{}c@{}}No. \\ Features\end{tabular} & SS (LR) & \begin{tabular}[c]{@{}c@{}}Whole Data \\ (LR)\end{tabular}  & XGBoost   & \multicolumn{1}{l|}{\begin{tabular}[c]{@{}l@{}}Clustering\\ (LR)\end{tabular}} \\ \hline

\begin{tabular}[c]{@{}l@{}}Online News \\ Popularity~\citep{fernandes_proactive_2015}\end{tabular}  
& 5900   & 2  & 59   & \textbf{72 (70)} & 65 (69)  &67 (70)   & 64 (62)    \\ \hline
\begin{tabular}[c]{@{}l@{}}AIDS Clinical \\ Trial~\citep{hammer_trial_1996}\end{tabular}                                                                        & 2139  & 2  & 22  &\textbf{ 83 (87)} & 78 (89)  & 77 (50)  & 78 (60)    \\ \hline
\begin{tabular}[c]{@{}l@{}}Drug Consum- \\ ption  Risk~\citep{fehrman_five_2017} \end{tabular}                                                                   & 1884   & 2  & 11    & \textbf{81 (80)} & 68 (76)   & 68 (62)   & 67 (70)        \\ \hline
\begin{tabular}[c]{@{}l@{}}Covid-19 \\ Mortality~\citep{nizri_covid19_2020}\end{tabular}                
& 1340  & 2  & 12  &\textbf{ 64 (66)} & 54 (56) & 56 (57)   & 53 (56)     \\ \hline

\begin{tabular}[c]{@{}l@{}}Credit Card \\ Default~\citep{yeh_comparisons_2009}\end{tabular}           
& 10000  & 2   & 24    & 74 (73) & 67 (66) & \textbf{77 (87)}   & 66 (73)            \\ \hline
\begin{tabular}[c]{@{}l@{}}Telco \\ Churn~\citep{ibm_telco_customer_churn}\end{tabular}                     
& 7043  & 2    & 20    & \textbf{83 (92)} & 80 (93)  & \textbf{84 (76)}   & 79 (75)      \\ \hline
Nursery~\citep{olave_manuel_application_2000}                                                                                     
& 12630  & 3   & 8   & 96 (98) & 93 (97) & \textbf{100 (100)} & 94 (98)              \\ \hline

\begin{tabular}[c]{@{}l@{}}Wine \\ Quality~\citep{paulo_cortez_wine_2009}\end{tabular}                   
& 6497 & 3  & 12  & 79 (61) & 78 (51)  & \textbf{84 (65)}   & 78 (74)          \\ \hline
\begin{tabular}[c]{@{}l@{}}Airline Passenger \\ Satisfaction~\citep{teejmahal20_airline_satisfaction_2021}\end{tabular} 
& 10000   & 2  & 22  & 89 (77) & 87 (89)  &\textbf{96 (99)}   & 88 (85)    \\ \hline
DNA~\citep{mlbench_R}                                                                                         
& 3186  & 3  & 180   & 92 (98) & 92 (98)  & \textbf{96 (98) }  & 87 (87)      \\ \hline
\begin{tabular}[c]{@{}l@{}}Online \\ Phishing~\citep{abdelhamid_phishing_2014}\end{tabular}              
& 1250    & 2   & 9   & 92 (92) & 92 (97)  & \textbf{93 (98)}   & 92 (92)   \\ \hline

\begin{tabular}[c]{@{}l@{}}Heart \\ Disease~\citep{detrano_international_1989}\end{tabular}                         
& 918  & 2      & 12   & 86 (67) & 86 (92)  & \textbf{89 (94)}   & 86 (82) \\ \hline
\begin{tabular}[c]{@{}l@{}}Obesity \\ Level~\citep{palechor_dataset_2019} \end{tabular}                                                                         & 2105    & 7   & 17   & 88 (84) & \textbf{94 (95)}  & \textbf{96 (99)}   & 84 (97)      \\ \hline
\begin{tabular}[c]{@{}l@{}}Student \\ Dropout~\citep{martins_early_2021}\end{tabular}                     
& 4424  & 3  & 35  & 67 (55) & \textbf{77 (74) } & \textbf{78 (77)}   & 71 (60)         \\ \hline

\end{tabular}
}
\caption{Comparison of classification performance between the Simple Structure (SS) model and standard machine learning approaches (Logistic Regression, XGBoost) across various datasets.}
\label{tab:UCI}
\end{table}

Our Simple Structure (SS) identification algorithm involves three primary hyperparameters: (\emph{i})~$K$, which determines the size of each local neighborhood; (\emph{ii})~the Simple Structure Threshold ($SST$), which specifies permissible structure sizes; and (\emph{iii})~the number of starting seeds that initiate the search for candidate structures. Although we do not provide a full sensitivity analysis here due to space constraints, we summarize the key considerations guiding our tuning process. A smaller $K$ produces very localized, potentially over-specialized structures, whereas a larger $K$ yields broader, more generic clusters that may reduce interpretability. Empirically, we observed that there is typically a narrow, dataset-dependent range of $K$ values that produces moderately sized structures. Moreover, we fix the number of starting seeds to 10\% of the dataset, since adding more seeds increases computational complexity without noticeably improving results; this proportion generally ensures sufficient coverage of the data to avoid discovering only very small or overly generic structures.
We introduced $SST$ to filter out extreme cases—whether excessively small or large—that may arise from the multi-seed process. In this work, we set $SST$ to 10\% of the dataset size, effectively discarding outlier structures while still allowing for a variety of meaningful subgroups. Because $SST$ depends on $K$, its viable range also tends to be narrow once smaller, redundant structures are removed. While our fixed choice of $SST$ may not be optimal for every dataset, the inherent penalty on unreasonably sized structures helps ensure that useful subgroups consistently emerge. Finally, adjusting the number of seeds and $SST$ can shift the order or nature of the identified structures, but once $K$ and $SST$ are reasonably tuned, the SS algorithm reliably recovers interpretable subgroups without significantly compromising accuracy.

Table~\ref{tab:UCI} demonstrates that our SS (LR) approach typically achieves classification accuracy and AUPRC scores that are on par with, and in several instances exceed, those of standard logistic regression (LR) using all features. Moreover, SS (LR) remains competitive against more complex methods such as XGBoost, occasionally outperforming the latter (e.g., in the COVID-19, Drug Addiction, and Online News Popularity datasets) despite XGBoost’s extensive hyperparameter tuning and flexible modeling capabilities.

\section{Discussion}\label{sec3}

Building on our experimental findings, this section explores the broader implications of adopting the Simple Structure (SS) identification algorithm and the reasons behind its robust performance—even when certain foundational assumptions are not strictly satisfied. 

In our synthetic and real data experiments, the SS approach maintains strong predictive accuracy, often matching or surpassing more complex models. 
The concept of simple structures and the algorithm designed to identify them are fundamentally grounded in Assumption 1 (simple structures are well separated) and Assumption 2 (no isolated subsets within simple structures). However, real-world datasets rarely conform perfectly to these assumptions, often exhibiting overlapping structures and inconsistencies. Despite these deviations, robustness experiments on synthetic data demonstrate that the SS model maintains stable performance, with only a moderate decline in accuracy and precision-recall metrics. This resilience suggests that the algorithm is capable of identifying meaningful substructures even when ideal theoretical conditions are not fully met, reinforcing its applicability in practical settings.
The effectiveness of centroid allocation, as depicted in figure \ref{fig:precision_recall_comparison} hinges on the Assumption that centroids provide a sufficiently accurate approximation of the underlying structures. If the identified simple structures centroids deviated significantly from the ground truth, its efficacy in consistently improving recall might not have been observed. The algorithm thus robustly identifies the underlying simple structures.

We also claim that identifying simple structures within data simplifies the learned model. Figure \ref{fig:simple_vs_whole} shows that a better way to handle data is to acknowledge the existence of two simple structures by learning a decision boundary at the root as $x<13$, with decision boundaries in either structure not interfering with adjacent structures. Top-down models won't learn such roots because they don't acknowledge the existence of simple structures, whereas our algorithm recognizes simple structures. This underscores the capability of our bottom-up simple structure approach to identify diverse local meaningful simple structures, unlike top-down models focused solely on reducing global loss.

The simple structures identified by our algorithm are not intended to serve as the final prediction rule but rather as a domain-relevant framework for partitioning the data into coherent subgroups. Each of these subgroups can then be handled by specialized, yet still interpretable, models such as small decision trees or logistic regression. This approach differs fundamentally from both single, global decision trees, which may overlook important variations across subgroups, and large, universal models like XGBoost or neural networks, which often optimize for predictive performance at the expense of interpretability. By first identifying meaningful, simple structures and then fitting lightweight, localized models within each, our method achieves a balance between subgroup-level interpretability and predictive performance, making it a compelling alternative to traditional black-box models.

Thus, similar to our approach on synthetic data, for real datasets in Table~\ref{tab:UCI} ensemble of localized logistic regression models yields comparable performance against complex models such as XGBoost and often outperforms them.
One of the key advantages of simple models such as logistic regression is their transparency: each feature’s coefficient directly represents its log-odds contribution, allowing practitioners to clearly understand how a given predictor influences the outcome while holding others constant. In contrast, models like XGBoost and neural networks introduce additional layers of complexity that obscure individual feature effects. XGBoost aggregates decisions across numerous trees, resulting in a diffuse, non-trivial prediction mechanism that cannot be easily distilled into a single interpretable formula. Similarly, neural networks rely on multiple hidden layers and nonlinear transformations, making their internal decision-making process inherently opaque. By leveraging simple structures to segment the data and applying logistic regression within each, our method preserves the clarity of interpretable models while achieving accuracy comparable to these more complex approaches.

Although clustering or GMM-based approaches are established alternatives to our method, our results show that the simple structures identified by our algorithm differ substantially from GMM-derived distributions. As illustrated in Figure \ref{fig: SS_vs_GMM }, a simple structure is not necessarily tied to a single distribution or class, whereas Gaussian Mixture Models (GMMs) rely on parametric assumptions (e.g., normality) and the Expectation-Maximization (EM) algorithm. GMMs can perform well when these assumptions hold but often degrade with non-normal data or multiple classes that do not fit neatly into Gaussian components. In contrast, our SS approach does not depend on any particular distribution, enabling it to detect meaningful substructures—even in heterogeneous settings—and robustly capture diverse underlying patterns.

Empirical results on real-life data (Table~\ref{tab:UCI}) indicate that our Simple Structure (SS) algorithm performs consistently well in domains with complex feature interactions and diverse behavioral patterns. Notably, in the first four datasets, our method—an ensemble of logistic regression models fitted on top of simple structures—demonstrates strong performance across both accuracy and AUPRC metrics.
Three of these four datasets involve behavioral patterns or precision medicine applications, where localized models naturally align with the data’s inherent complexity. For instance, COVID-19 treatment outcomes are heavily influenced by comorbidities, which affect patients differently~\citep{djaharuddin_comorbidities_2021}. In precision medicine, it is well established that treatments are not universally effective across entire populations—success and failure often depend on individual patterns that form distinct local structures~\citep{johnson_precision_2021}. Similarly, news popularity is shaped by reader preferences, which vary significantly across demographics, and drug usage or addiction follows individual behavioral tendencies. While each person may have unique characteristics, subsets of individuals share similar tendencies, allowing our SS algorithm to uncover and exploit these local patterns. By identifying and modeling these structures separately, SS often outperforms global methods, as it captures crucial distinctions that a single, overarching model might overlook.

In datasets such as credit card default, telco churn, nursery admission, wine quality, and airline satisfaction, SS outperforms global logistic regression but does not surpass XGBoost. This suggests that while some local structures exist, they are weaker or less critical to prediction. However, in settings where interpretability is as valuable as accuracy, such as customer churn analysis, our method provides a transparent alternative, helping decision-makers extract actionable insights from structured customer segments.

For datasets like heart failure, DNA classification, and online phishing detection, SS performs comparably to global logistic regression, indicating that strong simple structures may not exist. SS struggles most in datasets with globally simple structures, such as obesity prediction and student dropout analysis, where even a basic logistic regression model rivals XGBoost, making local segmentation and complex modelling unnecessary.
Overall, SS excels in datasets where localized patterns influence outcomes, allowing for specialized, interpretable decision boundaries while maintaining competitive accuracy. This makes SS particularly valuable in applications where both predictive performance and explainability are essential.

The Benefit of Simple Structure lies in discovering multiple local structures defined by a compact set of influential features, thereby reducing redundant, complex feature interactions within each subregion. Each simple structure uses only the features most relevant to that local subset, allowing for clearer decision boundaries and more interpretable results. Moreover, these local submodels jointly approximate the global data distribution more effectively than a single monolithic model, particularly in domains where nuanced, domain-specific interactions are crucial.

Overall, the SS algorithm’s bottom-up approach, combined with its ability to reliably identify meaningful local structures even when theoretical assumptions are mildly violated, makes it an attractive solution for tasks requiring both accuracy and interpretability.

\section{Conclusion and Future Work}

Our Simple Structure (SS) identification algorithm offers an appealing balance between accuracy, interpretability, and robustness across a range of datasets. By decomposing data into multiple local substructures, SS reveals context-specific patterns that global models often overlook. As demonstrated in our experiments, SS combined with logistic regression either outperforms or closely matches more complex learners—such as XGBoost or fully connected neural networks—while maintaining clearer and more intuitive decision rationales.

In contrast to clustering- or GMM-based approaches, SS does not rely on restrictive assumptions about data distributions, making it more adaptable to heterogeneous, real-world scenarios. Despite these promising outcomes, several avenues for future research remain. First, a more comprehensive sensitivity analysis of hyperparameters could further illuminate optimal strategies for diverse data environments. Second, the core algorithm itself can be modified or extended to handle more sophisticated forms of local structure, potentially boosting performance in especially high-dimensional or noisy domains. Finally, a detailed case study on real-world data applications would provide deeper insight into how SS can uncover and leverage meaningful substructures in practice. Taken together, these enhancements and validations would further establish SS as a valuable tool in domains where clarity and predictive power must go hand in hand.

\section{Data Availability}

All datasets used in this study are publicly available from open-source repositories. Specifically, datasets were obtained from Kaggle, the UCI Machine Learning Repository, or the MLBench library, as referenced in the citations.

\section{Compliance with Ethical Standards}

\begin{itemize}
\item  Disclosure of potential conflicts of interest -  Authors declare that they have no conflict of interest.
\item Research involving human participants and/or animals (Ethics Approval)  - This article does not contain any studies with human participants or animals performed by any of the authors.
\end{itemize}

\bibliography{Subsetting_algorithm}


\begin{thebibliography}{45}
\ifx \bisbn   \undefined \def \bisbn  #1{ISBN #1}\fi
\ifx \binits  \undefined \def \binits#1{#1}\fi
\ifx \bauthor  \undefined \def \bauthor#1{#1}\fi
\ifx \batitle  \undefined \def \batitle#1{#1}\fi
\ifx \bjtitle  \undefined \def \bjtitle#1{#1}\fi
\ifx \bvolume  \undefined \def \bvolume#1{\textbf{#1}}\fi
\ifx \byear  \undefined \def \byear#1{#1}\fi
\ifx \bissue  \undefined \def \bissue#1{#1}\fi
\ifx \bfpage  \undefined \def \bfpage#1{#1}\fi
\ifx \blpage  \undefined \def \blpage #1{#1}\fi
\ifx \burl  \undefined \def \burl#1{\textsf{#1}}\fi
\ifx \doiurl  \undefined \def \doiurl#1{\url{https://doi.org/#1}}\fi
\ifx \betal  \undefined \def \betal{\textit{et al.}}\fi
\ifx \binstitute  \undefined \def \binstitute#1{#1}\fi
\ifx \binstitutionaled  \undefined \def \binstitutionaled#1{#1}\fi
\ifx \bctitle  \undefined \def \bctitle#1{#1}\fi
\ifx \beditor  \undefined \def \beditor#1{#1}\fi
\ifx \bpublisher  \undefined \def \bpublisher#1{#1}\fi
\ifx \bbtitle  \undefined \def \bbtitle#1{#1}\fi
\ifx \bedition  \undefined \def \bedition#1{#1}\fi
\ifx \bseriesno  \undefined \def \bseriesno#1{#1}\fi
\ifx \blocation  \undefined \def \blocation#1{#1}\fi
\ifx \bsertitle  \undefined \def \bsertitle#1{#1}\fi
\ifx \bsnm \undefined \def \bsnm#1{#1}\fi
\ifx \bsuffix \undefined \def \bsuffix#1{#1}\fi
\ifx \bparticle \undefined \def \bparticle#1{#1}\fi
\ifx \barticle \undefined \def \barticle#1{#1}\fi
\bibcommenthead
\ifx \bconfdate \undefined \def \bconfdate #1{#1}\fi
\ifx \botherref \undefined \def \botherref #1{#1}\fi
\ifx \url \undefined \def \url#1{\textsf{#1}}\fi
\ifx \bchapter \undefined \def \bchapter#1{#1}\fi
\ifx \bbook \undefined \def \bbook#1{#1}\fi
\ifx \bcomment \undefined \def \bcomment#1{#1}\fi
\ifx \oauthor \undefined \def \oauthor#1{#1}\fi
\ifx \citeauthoryear \undefined \def \citeauthoryear#1{#1}\fi
\ifx \endbibitem  \undefined \def \endbibitem {}\fi
\ifx \bconflocation  \undefined \def \bconflocation#1{#1}\fi
\ifx \arxivurl  \undefined \def \arxivurl#1{\textsf{#1}}\fi
\csname PreBibitemsHook\endcsname

\bibitem[\protect\citeauthoryear{Ghadikolaei et~al.}{2019}]{ghadikolaei_learning_2019}
\begin{bchapter}
\bauthor{\bsnm{Ghadikolaei}, \binits{H.S.}},
\bauthor{\bsnm{Ghauch}, \binits{H.}},
\bauthor{\bsnm{Fischione}, \binits{C.}},
\bauthor{\bsnm{Skoglund}, \binits{M.}}:
\bctitle{Learning and data selection in big datasets}.
In: \bbtitle{Proceedings of the 36th International Conference on Machine Learning},
pp. \bfpage{2191}--\blpage{2200}.
\bpublisher{{PMLR}},
\blocation{Long Beach}
(\byear{2019})
\end{bchapter}
\endbibitem

\bibitem[\protect\citeauthoryear{Bishop}{2006}]{bishop_pattern_2006}
\begin{bbook}
\bauthor{\bsnm{Bishop}, \binits{C.M.}}:
\bbtitle{Pattern Recognition and Machine Learning}.
\bsertitle{Information science and statistics}.
\bpublisher{Springer},
\blocation{New York}
(\byear{2006})
\end{bbook}
\endbibitem

\bibitem[\protect\citeauthoryear{Hart}{1968}]{hart_condensed_1968}
\begin{barticle}
\bauthor{\bsnm{Hart}, \binits{P.}}:
\batitle{The condensed nearest neighbor rule (corresp.)}.
\bjtitle{{IEEE} Transactions on Information Theory}
\bvolume{14}(\bissue{3}),
\bfpage{515}--\blpage{516}
(\byear{1968})
\end{barticle}
\endbibitem

\bibitem[\protect\citeauthoryear{Gates}{1972}]{gates_reduced_1972}
\begin{barticle}
\bauthor{\bsnm{Gates}, \binits{G.}}:
\batitle{The reduced nearest neighbor rule (corresp.)}.
\bjtitle{{IEEE} Transactions on Information Theory}
\bvolume{18}(\bissue{3}),
\bfpage{431}--\blpage{433}
(\byear{1972})
\end{barticle}
\endbibitem

\bibitem[\protect\citeauthoryear{Carbonera and Abel}{2015}]{carbonera_density-based_2015}
\begin{bchapter}
\bauthor{\bsnm{Carbonera}, \binits{J.L.}},
\bauthor{\bsnm{Abel}, \binits{M.}}:
\bctitle{A density-based approach for instance selection}.
In: \bbtitle{2015 {IEEE} 27th International Conference on Tools with Artificial Intelligence ({ICTAI})},
pp. \bfpage{768}--\blpage{774}
(\byear{2015})
\end{bchapter}
\endbibitem

\bibitem[\protect\citeauthoryear{Cavalcanti and Soares}{2020}]{cavalcanti_ranking-based_2020}
\begin{barticle}
\bauthor{\bsnm{Cavalcanti}, \binits{G.D.C.}},
\bauthor{\bsnm{Soares}, \binits{R.J.O.}}:
\batitle{Ranking-based instance selection for pattern classification}.
\bjtitle{Expert Systems with Applications}
\bvolume{150},
\bfpage{113269}
(\byear{2020})
\end{barticle}
\endbibitem

\bibitem[\protect\citeauthoryear{Olvera-López et~al.}{2010}]{olvera-lopez_review_2010}
\begin{barticle}
\bauthor{\bsnm{Olvera-López}, \binits{J.A.}},
\bauthor{\bsnm{Carrasco-Ochoa}, \binits{J.A.}},
\bauthor{\bsnm{Martínez-Trinidad}, \binits{J.F.}},
\bauthor{\bsnm{Kittler}, \binits{J.}}:
\batitle{A review of instance selection methods}.
\bjtitle{Artificial Intelligence Review}
\bvolume{34}(\bissue{2}),
\bfpage{133}--\blpage{143}
(\byear{2010})
\end{barticle}
\endbibitem

\bibitem[\protect\citeauthoryear{Lin et~al.}{2021}]{lin_simultaneous_2021}
\begin{barticle}
\bauthor{\bsnm{Lin}, \binits{C.-C.}},
\bauthor{\bsnm{Kang}, \binits{J.-R.}},
\bauthor{\bsnm{Liang}, \binits{Y.-L.}},
\bauthor{\bsnm{Kuo}, \binits{C.-C.}}:
\batitle{Simultaneous feature and instance selection in big noisy data using memetic variable neighborhood search}.
\bjtitle{Applied Soft Computing}
\bvolume{112},
\bfpage{107855}
(\byear{2021})
\end{barticle}
\endbibitem

\bibitem[\protect\citeauthoryear{Neri and Triguero}{2020}]{neri_local_2020}
\begin{bchapter}
\bauthor{\bsnm{Neri}, \binits{F.}},
\bauthor{\bsnm{Triguero}, \binits{I.}}:
\bctitle{A local search with a surrogate assisted option for instance reduction}.
In: \bbtitle{Applications of Evolutionary Computation}.
\bsertitle{Lecture Notes in Computer Science},
pp. \bfpage{578}--\blpage{594}.
\bpublisher{Springer}, \blocation{???}
(\byear{2020})
\end{bchapter}
\endbibitem

\bibitem[\protect\citeauthoryear{Meir and Rätsch}{2003}]{meir_introduction_2003}
\begin{bchapter}
\bauthor{\bsnm{Meir}, \binits{R.}},
\bauthor{\bsnm{Rätsch}, \binits{G.}}:
\bctitle{An introduction to boosting and leveraging}.
In: \beditor{\bsnm{Mendelson}, \binits{S.}},
\beditor{\bsnm{Smola}, \binits{A.J.}} (eds.)
\bbtitle{Advanced Lectures on Machine Learning: Machine Learning Summer School 2002 Canberra, Australia},
pp. \bfpage{118}--\blpage{183}.
\bpublisher{Springer},
\blocation{Canberra}
(\byear{2003})
\end{bchapter}
\endbibitem

\bibitem[\protect\citeauthoryear{Freund and Schapire}{1997}]{freund_decision-theoretic_1997}
\begin{barticle}
\bauthor{\bsnm{Freund}, \binits{Y.}},
\bauthor{\bsnm{Schapire}, \binits{R.E.}}:
\batitle{A decision-theoretic generalization of on-line learning and an application to boosting}.
\bjtitle{Journal of Computer and System Sciences}
\bvolume{55}(\bissue{1}),
\bfpage{119}--\blpage{139}
(\byear{1997})
\end{barticle}
\endbibitem

\bibitem[\protect\citeauthoryear{Friedman}{2001}]{friedman_greedy_2001}
\begin{barticle}
\bauthor{\bsnm{Friedman}, \binits{J.H.}}:
\batitle{Greedy function approximation: A gradient boosting machine.}
\bjtitle{The Annals of Statistics}
\bvolume{29}(\bissue{5}),
\bfpage{1189}--\blpage{1232}
(\byear{2001})
\end{barticle}
\endbibitem

\bibitem[\protect\citeauthoryear{Chen and Guestrin}{2016}]{chen_xgboost_2016}
\begin{bchapter}
\bauthor{\bsnm{Chen}, \binits{T.}},
\bauthor{\bsnm{Guestrin}, \binits{C.}}:
\bctitle{{XGBoost}: A scalable tree boosting system}.
In: \bbtitle{Proceedings of the 22nd {ACM} {SIGKDD} International Conference on Knowledge Discovery and Data Mining}.
\bsertitle{{KDD} '16},
pp. \bfpage{785}--\blpage{794}.
\bpublisher{Association for Computing Machinery},
\blocation{San Fransisco}
(\byear{2016})
\end{bchapter}
\endbibitem

\bibitem[\protect\citeauthoryear{Caruana et~al.}{2015}]{caruana_intelligible_2015}
\begin{bchapter}
\bauthor{\bsnm{Caruana}, \binits{R.}},
\bauthor{\bsnm{Lou}, \binits{Y.}},
\bauthor{\bsnm{Gehrke}, \binits{J.}},
\bauthor{\bsnm{Koch}, \binits{P.}},
\bauthor{\bsnm{Sturm}, \binits{M.}},
\bauthor{\bsnm{Elhadad}, \binits{N.}}:
\bctitle{Intelligible {Models} for {HealthCare}: {Predicting} {Pneumonia} {Risk} and {Hospital} 30-day {Readmission}}.
In: \bbtitle{Proceedings of the 21th {ACM} {SIGKDD} {International} {Conference} on {Knowledge} {Discovery} And {Data} {Mining}},
pp. \bfpage{1721}--\blpage{1730}.
\bpublisher{Association for Computing Machinery},
\blocation{Sydney}
(\byear{2015})
\end{bchapter}
\endbibitem

\bibitem[\protect\citeauthoryear{Bank et~al.}{2021}]{bank_autoencoders_2020}
\begin{botherref}
\oauthor{\bsnm{Bank}, \binits{D.}},
\oauthor{\bsnm{Koenigstein}, \binits{N.}},
\oauthor{\bsnm{Giryes}, \binits{R.}}:
Autoencoders
(2021)
{\href{https://arxiv.org/abs/2003.05991}{{arXiv:2003.05991}}}
{[cs.LG]}
\end{botherref}
\endbibitem

\bibitem[\protect\citeauthoryear{Bengio et~al.}{2013}]{bengio_representation_2013}
\begin{barticle}
\bauthor{\bsnm{Bengio}, \binits{Y.}},
\bauthor{\bsnm{Courville}, \binits{A.}},
\bauthor{\bsnm{Vincent}, \binits{P.}}:
\batitle{Representation learning: A review and new perspectives}.
\bjtitle{{IEEE} Transactions on Pattern Analysis and Machine Intelligence}
\bvolume{35}(\bissue{8}),
\bfpage{1798}--\blpage{1828}
(\byear{2013})
\end{barticle}
\endbibitem

\bibitem[\protect\citeauthoryear{Chung and Weng}{2017}]{chung_learning_2017}
\begin{botherref}
\oauthor{\bsnm{Chung}, \binits{Y.-A.}},
\oauthor{\bsnm{Weng}, \binits{W.-H.}}:
Learning Deep Representations of Medical Images using Siamese CNNs with Application to Content-Based Image Retrieval
(2017)
\end{botherref}
\endbibitem

\bibitem[\protect\citeauthoryear{Tschannen et~al.}{2018}]{tschannen_recent_2018}
\begin{botherref}
\oauthor{\bsnm{Tschannen}, \binits{M.}},
\oauthor{\bsnm{Bachem}, \binits{O.}},
\oauthor{\bsnm{Lucic}, \binits{M.}}:
Recent advances in autoencoder-based representation learning
(2018)
{\href{https://arxiv.org/abs/1812.05069}{{arXiv:1812.05069}}}
{[cs.LG]}
\end{botherref}
\endbibitem

\bibitem[\protect\citeauthoryear{Su and Khoshgoftaar}{2009}]{su_survey_2009}
\begin{botherref}
\oauthor{\bsnm{Su}, \binits{X.}},
\oauthor{\bsnm{Khoshgoftaar}, \binits{T.M.}}:
A survey of collaborative filtering techniques.
Advances in Artificial Intelligence
\textbf{2009}
(2009)
\end{botherref}
\endbibitem

\bibitem[\protect\citeauthoryear{He et~al.}{2017}]{he_neural_2017}
\begin{bchapter}
\bauthor{\bsnm{He}, \binits{X.}},
\bauthor{\bsnm{Liao}, \binits{L.}},
\bauthor{\bsnm{Zhang}, \binits{H.}},
\bauthor{\bsnm{Nie}, \binits{L.}},
\bauthor{\bsnm{Hu}, \binits{X.}},
\bauthor{\bsnm{Chua}, \binits{T.-S.}}:
\bctitle{Neural collaborative filtering}.
In: \bbtitle{Proceedings of the 26th International Conference on World Wide Web},
pp. \bfpage{173}--\blpage{182}.
\bpublisher{International World Wide Web Conferences Steering Committee},
\blocation{Geneva}
(\byear{2017})
\end{bchapter}
\endbibitem

\bibitem[\protect\citeauthoryear{Elkahky et~al.}{2015}]{elkahky_multi-view_2015}
\begin{bchapter}
\bauthor{\bsnm{Elkahky}, \binits{A.M.}},
\bauthor{\bsnm{Song}, \binits{Y.}},
\bauthor{\bsnm{He}, \binits{X.}}:
\bctitle{A multi-view deep learning approach for cross domain user modeling in recommendation systems}.
In: \bbtitle{Proceedings of the 24th International Conference on World Wide Web},
pp. \bfpage{278}--\blpage{288}.
\bpublisher{International World Wide Web Conferences Steering Committee},
\blocation{New York}
(\byear{2015})
\end{bchapter}
\endbibitem

\bibitem[\protect\citeauthoryear{Zhang et~al.}{2019}]{zhang_deep_2019}
\begin{barticle}
\bauthor{\bsnm{Zhang}, \binits{S.}},
\bauthor{\bsnm{Yao}, \binits{L.}},
\bauthor{\bsnm{Sun}, \binits{A.}},
\bauthor{\bsnm{Tay}, \binits{Y.}}:
\batitle{Deep learning based recommender system: A survey and new perspectives}.
\bjtitle{{ACM} Computing Surveys}
\bvolume{52}(\bissue{1}),
\bfpage{5}--\blpage{1538}
(\byear{2019})
\end{barticle}
\endbibitem

\bibitem[\protect\citeauthoryear{Wan et~al.}{2019}]{wan_novel_nodate}
\begin{bchapter}
\bauthor{\bsnm{Wan}, \binits{H.}},
\bauthor{\bsnm{Wang}, \binits{H.}},
\bauthor{\bsnm{Scotney}, \binits{B.}},
\bauthor{\bsnm{Liu}, \binits{J.}}:
\bctitle{A novel gaussian mixture model for classification}.
In: \bbtitle{2019 IEEE International Conference on Systems, Man and Cybernetics (SMC)},
pp. \bfpage{3298}--\blpage{3303}
(\byear{2019})
\end{bchapter}
\endbibitem

\bibitem[\protect\citeauthoryear{Colijn et~al.}{2017}]{colijn_toward_2017}
\begin{botherref}
\oauthor{\bsnm{Colijn}, \binits{C.}},
\oauthor{\bsnm{Jones}, \binits{N.}},
\oauthor{\bsnm{Johnston}, \binits{I.G.}},
\oauthor{\bsnm{Yaliraki}, \binits{S.}},
\oauthor{\bsnm{Barahona}, \binits{M.}}:
Toward precision healthcare: Context and mathematical challenges.
Frontiers in Physiology
\textbf{8}
(2017)
\end{botherref}
\endbibitem

\bibitem[\protect\citeauthoryear{Tomek}{1976}]{noauthor_experiment_1976}
\begin{barticle}
\bauthor{\bsnm{Tomek}, \binits{I.}}:
\batitle{An experiment with the edited nearest-neighbor rule}.
\bjtitle{{IEEE} Transactions on Systems, Man, and Cybernetics}
\bvolume{{SMC}-6}(\bissue{6}),
\bfpage{448}--\blpage{452}
(\byear{1976})
\end{barticle}
\endbibitem

\bibitem[\protect\citeauthoryear{Na et~al.}{2010}]{na_research_2010}
\begin{bchapter}
\bauthor{\bsnm{Na}, \binits{S.}},
\bauthor{\bsnm{Xumin}, \binits{L.}},
\bauthor{\bsnm{Yong}, \binits{G.}}:
\bctitle{Research on k-means clustering algorithm: An improved k-means clustering algorithm}.
In: \bbtitle{2010 Third International Symposium on Intelligent Information Technology and Security Informatics},
pp. \bfpage{63}--\blpage{67}
(\byear{2010})
\end{bchapter}
\endbibitem

\bibitem[\protect\citeauthoryear{Wan et~al.}{2019}]{wan_novel_2019}
\begin{bchapter}
\bauthor{\bsnm{Wan}, \binits{H.}},
\bauthor{\bsnm{Wang}, \binits{H.}},
\bauthor{\bsnm{Scotney}, \binits{B.}},
\bauthor{\bsnm{Liu}, \binits{J.}}:
\bctitle{A {Novel} {Gaussian} {Mixture} {Model} for {Classification}}.
In: \bbtitle{2019 {IEEE} {International} {Conference} on {Systems}, {Man} and {Cybernetics} ({SMC})},
pp. \bfpage{3298}--\blpage{3303}
(\byear{2019})
\end{bchapter}
\endbibitem

\bibitem[\protect\citeauthoryear{Ahmad and Khan}{2019}]{ahmad_survey_2019}
\begin{barticle}
\bauthor{\bsnm{Ahmad}, \binits{A.}},
\bauthor{\bsnm{Khan}, \binits{S.S.}}:
\batitle{Survey of {State}-of-the-{Art} {Mixed} {Data} {Clustering} {Algorithms}}.
\bjtitle{IEEE Access}
\bvolume{7},
\bfpage{31883}--\blpage{31902}
(\byear{2019})
\end{barticle}
\endbibitem

\bibitem[\protect\citeauthoryear{Kaufman and Rousseeuw}{1990}]{Kaufman_pam_1990}
\begin{bbook}
\bauthor{\bsnm{Kaufman}, \binits{L.}},
\bauthor{\bsnm{Rousseeuw}, \binits{P.}}:
\bbtitle{Finding Groups in Data: An Introduction To Cluster Analysis}.
\bpublisher{Wiley, New York. ISBN 0-471-87876-6.}, \blocation{???}
(\byear{1990}).
\doiurl{10.2307/2532178}
\end{bbook}
\endbibitem

\bibitem[\protect\citeauthoryear{Fernandes et~al.}{2015}]{fernandes_proactive_2015}
\begin{bchapter}
\bauthor{\bsnm{Fernandes}, \binits{K.}},
\bauthor{\bsnm{Vinagre}, \binits{P.}},
\bauthor{\bsnm{Cortez}, \binits{P.}}:
\bctitle{A {Proactive} {Intelligent} {Decision} {Support} {System} for {Predicting} the {Popularity} of {Online} {News}}.
In: \beditor{\bsnm{Pereira}, \binits{F.}},
\beditor{\bsnm{Machado}, \binits{P.}},
\beditor{\bsnm{Costa}, \binits{E.}},
\beditor{\bsnm{Cardoso}, \binits{A.}} (eds.)
\bbtitle{Progress in {Artificial} {Intelligence}},
pp. \bfpage{535}--\blpage{546}.
\bpublisher{Springer},
\blocation{Cham}
(\byear{2015})
\end{bchapter}
\endbibitem

\bibitem[\protect\citeauthoryear{Hammer et~al.}{1996}]{hammer_trial_1996}
\begin{barticle}
\bauthor{\bsnm{Hammer}, \binits{S.M.}},
\bauthor{\bsnm{Katzenstein}, \binits{D.A.}},
\bauthor{\bsnm{Hughes}, \binits{M.D.}},
\bauthor{\bsnm{Gundacker}, \binits{H.}},
\bauthor{\bsnm{Schooley}, \binits{R.T.}},
\bauthor{\bsnm{Haubrich}, \binits{R.H.}},
\bauthor{\bsnm{Henry}, \binits{W.K.}},
\bauthor{\bsnm{Lederman}, \binits{M.M.}},
\bauthor{\bsnm{Phair}, \binits{J.P.}},
\bauthor{\bsnm{Niu}, \binits{M.}},
\bauthor{\bsnm{Hirsch}, \binits{M.S.}},
\bauthor{\bsnm{Merigan}, \binits{T.C.}}:
\batitle{A {Trial} {Comparing} {Nucleoside} {Monotherapy} with {Combination} {Therapy} in {HIV}-{Infected} {Adults} with {CD4} {Cell} {Counts} from 200 to 500 per {Cubic} {Millimeter}}.
\bjtitle{New England Journal of Medicine}
\bvolume{335}(\bissue{15}),
\bfpage{1081}--\blpage{1090}
(\byear{1996})
\end{barticle}
\endbibitem

\bibitem[\protect\citeauthoryear{Fehrman et~al.}{2017}]{fehrman_five_2017}
\begin{botherref}
\oauthor{\bsnm{Fehrman}, \binits{E.}},
\oauthor{\bsnm{Muhammad}, \binits{A.K.}},
\oauthor{\bsnm{Mirkes}, \binits{E.M.}},
\oauthor{\bsnm{Egan}, \binits{V.}},
\oauthor{\bsnm{Gorban}, \binits{A.N.}}:
The {Five} {Factor} {Model} of personality and evaluation of drug consumption risk.
arXiv
(2017).
\doiurl{10.48550/arXiv.1506.06297}
\end{botherref}
\endbibitem

\bibitem[\protect\citeauthoryear{Nizri}{2020}]{nizri_covid19_2020}
\begin{botherref}
\oauthor{\bsnm{Nizri}, \binits{M.}}:
COVID-19 Dataset.
Kaggle.
\url{https://www.kaggle.com/datasets/meirnizri/covid19-dataset}
(2020)
\end{botherref}
\endbibitem

\bibitem[\protect\citeauthoryear{Yeh and Lien}{2009}]{yeh_comparisons_2009}
\begin{barticle}
\bauthor{\bsnm{Yeh}, \binits{I.-C.}},
\bauthor{\bsnm{Lien}, \binits{C.-h.}}:
\batitle{The comparisons of data mining techniques for the predictive accuracy of probability of default of credit card clients}.
\bjtitle{Expert Systems with Applications}
\bvolume{36}(\bissue{2, Part 1}),
\bfpage{2473}--\blpage{2480}
(\byear{2009})
\end{barticle}
\endbibitem

\bibitem[\protect\citeauthoryear{IBM}{2018}]{ibm_telco_customer_churn}
\begin{botherref}
\oauthor{\bsnm{IBM}}:
Telco Customer Churn.
IBM Sample Data Sets.
\url{https://www.kaggle.com/datasets/blastchar/telco-customer-churn}
(2018)
\end{botherref}
\endbibitem

\bibitem[\protect\citeauthoryear{Vladislav}{1989}]{olave_manuel_application_2000}
\begin{botherref}
\oauthor{\bsnm{Vladislav}, \binits{R.}}:
{Nursery}.
UCI Machine Learning Repository.
{DOI}: https://doi.org/10.24432/C5P88W
(1989)
\end{botherref}
\endbibitem

\bibitem[\protect\citeauthoryear{Cortez et~al.}{2009}]{paulo_cortez_wine_2009}
\begin{botherref}
\oauthor{\bsnm{Cortez}, \binits{P.}},
\oauthor{\bsnm{Cerdeira}, \binits{A.}},
\oauthor{\bsnm{Almeida}, \binits{F.}},
\oauthor{\bsnm{Matos}, \binits{T.}},
\oauthor{\bsnm{Reis}, \binits{J.}}:
Wine {Quality}.
UCI Machine Learning Repository
(2009).
\doiurl{10.24432/C56S3T}
\end{botherref}
\endbibitem

\bibitem[\protect\citeauthoryear{Klein}{2021}]{teejmahal20_airline_satisfaction_2021}
\begin{botherref}
\oauthor{\bsnm{Klein}, \binits{T.}}:
Airline Passenger Satisfaction.
Kaggle.
\url{https://www.kaggle.com/datasets/teejmahal20/airline-passenger-satisfaction}
(2021)
\end{botherref}
\endbibitem

\bibitem[\protect\citeauthoryear{Leisch and Dimitriadou}{2021}]{mlbench_R}
\begin{botherref}
\oauthor{\bsnm{Leisch}, \binits{F.}},
\oauthor{\bsnm{Dimitriadou}, \binits{E.}}:
mlbench: Machine Learning Benchmark Problems.
R package version 2.1-3.
\url{https://CRAN.R-project.org/package=mlbench}
(2021)
\end{botherref}
\endbibitem

\bibitem[\protect\citeauthoryear{Abdelhamid et~al.}{2014}]{abdelhamid_phishing_2014}
\begin{barticle}
\bauthor{\bsnm{Abdelhamid}, \binits{N.}},
\bauthor{\bsnm{Ayesh}, \binits{A.}},
\bauthor{\bsnm{Thabtah}, \binits{F.}}:
\batitle{Phishing detection based {Associative} {Classification} data mining}.
\bjtitle{Expert Systems with Applications}
\bvolume{41}(\bissue{13}),
\bfpage{5948}--\blpage{5959}
(\byear{2014})
\end{barticle}
\endbibitem

\bibitem[\protect\citeauthoryear{Detrano et~al.}{1989}]{detrano_international_1989}
\begin{barticle}
\bauthor{\bsnm{Detrano}, \binits{R.}},
\bauthor{\bsnm{Janosi}, \binits{A.}},
\bauthor{\bsnm{Steinbrunn}, \binits{W.}},
\bauthor{\bsnm{Pfisterer}, \binits{M.}},
\bauthor{\bsnm{Schmid}, \binits{J.-J.}},
\bauthor{\bsnm{Sandhu}, \binits{S.}},
\bauthor{\bsnm{Guppy}, \binits{K.H.}},
\bauthor{\bsnm{Lee}, \binits{S.}},
\bauthor{\bsnm{Froelicher}, \binits{V.}}:
\batitle{International application of a new probability algorithm for the diagnosis of coronary artery disease}.
\bjtitle{The American Journal of Cardiology}
\bvolume{64}(\bissue{5}),
\bfpage{304}--\blpage{310}
(\byear{1989})
\end{barticle}
\endbibitem

\bibitem[\protect\citeauthoryear{Palechor and Manotas}{2019}]{palechor_dataset_2019}
\begin{barticle}
\bauthor{\bsnm{Palechor}, \binits{F.M.}},
\bauthor{\bsnm{Manotas}, \binits{A.d.l.H.}}:
\batitle{Dataset for estimation of obesity levels based on eating habits and physical condition in individuals from {Colombia}, {Peru} and {Mexico}}.
\bjtitle{Data in Brief}
\bvolume{25},
\bfpage{104344}
(\byear{2019})
\end{barticle}
\endbibitem

\bibitem[\protect\citeauthoryear{Martins et~al.}{2021}]{martins_early_2021}
\begin{bchapter}
\bauthor{\bsnm{Martins}, \binits{M.V.}},
\bauthor{\bsnm{Tolledo}, \binits{D.}},
\bauthor{\bsnm{Machado}, \binits{J.}},
\bauthor{\bsnm{Baptista}, \binits{L.M.T.}},
\bauthor{\bsnm{Realinho}, \binits{V.}}:
\bctitle{Early {Prediction} of student’s {Performance} in {Higher} {Education}: {A} {Case} {Study}}.
In: \beditor{\bsnm{Rocha}, \binits{A.}},
\beditor{\bsnm{Adeli}, \binits{H.}},
\beditor{\bsnm{Dzemyda}, \binits{G.}},
\beditor{\bsnm{Moreira}, \binits{F.}},
\beditor{\bsnm{Ramalho~Correia}, \binits{A.M.}} (eds.)
\bbtitle{Trends and {Applications} in {Information} {Systems} and {Technologies}},
pp. \bfpage{166}--\blpage{175}.
\bpublisher{Springer},
\blocation{Cham}
(\byear{2021})
\end{bchapter}
\endbibitem

\bibitem[\protect\citeauthoryear{Djaharuddin et~al.}{2021}]{djaharuddin_comorbidities_2021}
\begin{barticle}
\bauthor{\bsnm{Djaharuddin}, \binits{I.}},
\bauthor{\bsnm{Munawwarah}, \binits{S.}},
\bauthor{\bsnm{Nurulita}, \binits{A.}},
\bauthor{\bsnm{Ilyas}, \binits{M.}},
\bauthor{\bsnm{Tabri}, \binits{N.A.}},
\bauthor{\bsnm{Lihawa}, \binits{N.}}:
\batitle{Comorbidities and mortality in {COVID}-19 patients}.
\bjtitle{Gaceta Sanitaria}
\bvolume{35},
\bfpage{530}--\blpage{532}
(\byear{2021})
\doiurl{10.1016/j.gaceta.2021.10.085}
\end{barticle}
\endbibitem

\bibitem[\protect\citeauthoryear{Johnson et~al.}{2021}]{johnson_precision_2021}
\begin{barticle}
\bauthor{\bsnm{Johnson}, \binits{K.B.}},
\bauthor{\bsnm{Wei}, \binits{W.-Q.}},
\bauthor{\bsnm{Weeraratne}, \binits{D.}},
\bauthor{\bsnm{Frisse}, \binits{M.E.}},
\bauthor{\bsnm{Misulis}, \binits{K.}},
\bauthor{\bsnm{Rhee}, \binits{K.}},
\bauthor{\bsnm{Zhao}, \binits{J.}},
\bauthor{\bsnm{Snowdon}, \binits{J.L.}}:
\batitle{Precision {Medicine}, {AI}, and the {Future} of {Personalized} {Health} {Care}}.
\bjtitle{Clinical and Translational Science}
\bvolume{14}(\bissue{1}),
\bfpage{86}--\blpage{93}
(\byear{2021})
\end{barticle}
\endbibitem

\end{thebibliography}

\end{document}